\def\year{2018}\relax
\documentclass[letterpaper]{article}
\usepackage{aaai18} 
\usepackage{times}
\usepackage{helvet}
\usepackage{courier}
\usepackage{amsmath}
\usepackage{algorithm}
\usepackage{amssymb}
\usepackage{url}

\usepackage{graphicx}
\usepackage{color}
\usepackage{listings}
\usepackage{wrapfig}

\newcommand{\cred}{\color{red}}

\long\def\BOC#1\EOC{\message{(Commented text )}}
\long\def\BOCC#1\EOCC{\message{(Commented text )}}
\long\def\BOCCC#1\EOCCC{\message{(Commented text )}}
\long\def\optional#1{\empty}

\long\def\NBB#1{}

\def\o{\overline}
\def\ar{\leftarrow}
\def\bi{\begin{itemize}}
\def\ei{\end{itemize}}
\def\beq{\begin{equation}}
\def\eeq#1{\label{#1}\end{equation}}
\def\ba{\begin{array}}
\def\ea{\end{array}}
\def\i#1{\hbox{\it #1\/}}

\def\sm{\rm SM}

\def\no{\i{not}}

\def\ar{\leftarrow}
\def\rar{\rightarrow}

\def\no{\i{not}}

\def\false{\hbox{\sc false}}
\def\true{\hbox{\sc true}}

\def\i#1{\hbox{\itshape #1\/}}

\def\qed{\quad \vrule height7.5pt width4.17pt depth0pt \medskip}

\DeclareSymbolFont{AMSa}{U}{msa}{m}{n}
\DeclareMathSymbol{\square}{\mathord}{AMSa}{"03}

\def\mu#1{\mathit{\underline{#1}}}
\def\lpmln{{\rm LP}^{\rm{MLN}}}

\def\bi{\begin{itemize}}
\def\ei{\end{itemize}}

\newtheorem{prop}{Proposition}
\newtheorem{thm}{Theorem}

\newtheorem{lemma}{Lemma} 
\newtheorem{example}{Example}

\begin{document}

\title{Weight Learning in a Probabilistic Extension of Answer Set Programs} 

\author{Joohyung Lee and Yi Wang\\ 
School of Computing, Informatics and Decision Systems Engineering \\
Arizona State University, Tempe, AZ, USA \\
{\tt \{joolee, ywang485\}@asu.edu}
}

\maketitle

\begin{abstract}
$\lpmln$ is a probabilistic extension of answer set programs with the weight scheme derived from that of Markov Logic. Previous work has shown how inference in $\lpmln$ can be achieved. In this paper, we present the concept of weight learning in $\lpmln$ and learning algorithms for $\lpmln$ derived from those for Markov Logic. We also present a prototype implementation that uses answer set solvers for learning as well as some example domains that illustrate distinct features of $\lpmln$ learning. Learning in $\lpmln$ is in accordance with the stable model semantics, thereby it learns parameters for probabilistic extensions of knowledge-rich domains where answer set programming has shown to be useful but limited to the deterministic case, such as reachability analysis and reasoning about actions in dynamic domains. We also apply the method to learn the parameters for probabilistic abductive reasoning about actions. 
\end{abstract}

\section{Introduction} 

$\lpmln$ is a probabilistic extension of answer set programs with the weight scheme derived from that of Markov Logic \cite{richardson06markov}. The language turns out to be highly expressive to embed several other probabilistic logic languages, such as P-log \cite{baral09probabilistic}, ProbLog \cite{deraedt07problog}, Markov Logic, and Causal Models \cite{pearl00causality}, as described in  \cite{lee15markov,lee16weighted,balai16ontherelationship,lee17lpmln}. Inference engines for $\lpmln$, such as {\sc lpmln2asp}, {\sc lpmln2mln} \cite{lee17computing}, and {\sc lpmln-models} \cite{wang17parallel}, have been developed based on the reduction of $\lpmln$ to answer set programs and Markov Logic. 

The weight associated with each $\lpmln$ rule roughly asserts how important the rule is in deriving a stable model. It can be manually specified by the user, which may be okay for a simple program, but a systematic assignment of weights for a complex program could be challenging. A solution would be to learn the weights automatically from the observed data.

With this goal in mind, this paper presents the concept of weight learning in $\lpmln$ and a few learning methods for $\lpmln$ derived from learning in Markov Logic. 
Weight learning in $\lpmln$ is to find the weights of the rules in the $\lpmln$ program such that the likelihood of the observed data according to the $\lpmln$ semantics is maximized, which is commonly known as Maximum Likelihood Estimation (MLE) in the practice of machine learning.

In $\lpmln$, due to the requirement of a stable model, deterministic dependencies are frequent. \citeauthor{poon06sound} (\citeyear{poon06sound}) noted that deterministic dependencies break the support of a probability distribution into disconnected regions, making it difficult to design ergodic Markov chains for  Markov Chain Monte Carlo (MCMC) sampling, which motivated them to develop an algorithm called {\em MC-SAT} that uses a satisfiability solver to find modes for computing conditional probabilities. Thanks to the close relationship between Markov Logic and $\lpmln$, we could adapt that algorithm to $\lpmln$, which we call {\em MC-ASP}. Unlike MC-SAT, algorithm MC-ASP utilizes ASP solvers for performing MCMC sampling, and is based on the penalty-based formulation of $\lpmln$ instead of the reward-based formulation as in Markov Logic.

Learning in $\lpmln$ is in accordance with the stable model semantics, thereby it learns parameters for probabilistic extensions of knowledge-rich domains where answer set programming has shown to be useful but limited to the deterministic case, such as reachability analysis and reasoning about actions in dynamic domains. More interestingly, we demonstrate that the method can also be applied to learn parameters for abductive reasoning about dynamic systems to associate the probability learned from data with each possible reason for the failure. 

The paper is organized as follows. Section~\ref{sec:review} reviews the language $\lpmln$, and Section~\ref{sec:lpmln-weight-learning} presents the learning framework and a gradient ascent method for the basic case, where a single stable model is given as the training data. Section~\ref{sec:extensions} presents a few extensions of the learning problem and methods, such as allowing multiple stable models as the training data and allowing the training data to be an incomplete interpretation. In addition to the general learning algorithm, Section~\ref{sec:learning-other} relates $\lpmln$ learning also to learning in ProbLog and Markov Logic as special cases, which allows for the special cases of $\lpmln$ learning to be computed by existing implementations of ProbLog and Markov Logic. Section~\ref{sec:examples}  introduces a prototype implementation of the general learning algorithm and demonstrates it with a few example domains where $\lpmln$ learning is more suitable than other learning methods. 

\section{Review: Language $\lpmln$}\label{sec:review}

The original definition of $\lpmln$ from~\cite{lee16weighted} is based on the concept of a ``reward'': the more rules are true, the larger weight is assigned to the corresponding stable model as the reward.  Alternatively, ~\citeauthor{lee17lpmln} [\citeyear{lee17lpmln}] present a reformulation in terms of a``penalty'': the more rules are false, the smaller weight is assigned to the corresponding stable model. The advantage of the latter is that it yields a translation of $\lpmln$ programs that can be readily accepted by ASP solvers, the idea that led to the implementation of $\lpmln$ using ASP solvers  \cite{lee17computing}.
Throughout the paper, we refer to this reformulation as the main definition of $\lpmln$.

We assume a first-order signature~$\sigma$ that contains no function constants of positive arity, which yields finitely many Herbrand interpretations. 
An $\lpmln$ program is a pair $\langle {\bf R}, {\bf W}\rangle$, where ${\bf R}$ is a list of rules $(R_1, \dots, R_m)$, where each rule has the form
\beq
\ba l
    A\ar B\land N
\ea 
\eeq{rule0}
where $A$ is a disjunction of atoms, $B$ is a conjunction of atoms, and $N$ is a negative formula constructed from atoms using conjunction, disjunction, and negation.\footnote{For the definition of a negative formula, see~\cite{ferraris11stable}.}
We identify rule~\eqref{rule0} with formula $B\land N\rar A$.
The expression $\{A_1\}\ar \i{Body}$, where $A_1$ is an atom, denotes the rule $A_1\ar \i{Body}\land \neg\neg A_1$. 
${\bf W}$ is a list $(w_1,\dots w_m)$ such that each $w_i$ is a real number or the symbol $\alpha$ that denotes the weight of rule $i$ in ${\bf R}$. 
We can also identify an $\lpmln$ program with the finite list of weighted rules 
$
  \{w_i: R_i \mid i\in\{1, \dots, m\}\}.
  $
A weighted rule $w:R$ is called {\em soft} if $w$ is a real number; it is called {\em hard} if $w$ is $\alpha$ (which denotes infinite weight). 
Variables range over an Herbrand Universe, which is assumed to be finite so that the ground program is finite.
For any $\lpmln$ program $\Pi$, by $gr(\Pi)$ we denote the program obtained from $\Pi$ by the process of grounding. 
Each resulting rule with no variables, which we call {\em ground instance}, receives the same weight as the original rule.

For any $\lpmln$ program $\Pi=\{w_1: R_1, \dots, w_m: R_m\}$ and any interpretation~$I$, expression $n_i(I)$ denotes the number of ground instances of $R_i$ that is false in $I$, and $\overline{\Pi}$ denotes the set of  (unweighted) formulas obtained from $\Pi$ by dropping the weight of every rule. When $\Pi$ has no variables, ${\Pi}_I$ denotes the set of weighted rules $w: R$ in $\Pi$ such that $I\models R$. 

In general, an $\lpmln$ program may even have stable models that violate some hard rules, which encode definite knowledge. However, throughout the paper, we restrict attention to $\lpmln$ programs whose stable models do not violate hard rules. 
More precisely, given an $\lpmln$ program $\Pi$, $\sm[\Pi]$ denotes the set
\[
\ba l
\{I\mid \text{$I$ is a (deterministic) stable model of $\o{gr(\Pi)_I}$} \\ 
\text{~~~~~~~~that satisfies all hard rules in $gr(\Pi)$} \} .
\ea
\]
For any interpretation $I$, its weight $W_{\Pi}(I)$ and its probability $P_\Pi(I)$ are defined as follows.
{\small
\[
W_\Pi(I) =
\begin{cases}
  exp\Bigg(-\sum\limits_{w_i:R_i\;\in\; \Pi^{\rm soft}} w_i n_i(I)\Bigg) & 
      \text{if $I\in\sm[\Pi]$}; \\
  0 & \text{otherwise},
\end{cases}
\]
}
where $\Pi^{\rm soft}$ consists of all soft rules in $\Pi$, and
\[
P_\Pi(I) = 
     \frac{W_{\Pi}(I)}
         {\underset{J\in{\rm SM}[\Pi]}{\sum}{W_{\Pi}(J)}}.
\]
An interpretation $I$ is called a {\em (probabilistic) stable model} of~$\Pi$ if $P_\Pi(I)\ne 0$.
When $\sm[\Pi]$ is non-empty, it turns out that every probabilistic stable model satisfies all hard rules, and the definitions of $W_\Pi(I)$ and $P_\Pi(I)$ above are equivalent to the original definitions \cite[Proposition~2]{lee16weighted}.

For any proposition $A$, the probability of $A$ under $\Pi$ is defined as
$
P_{\Pi}(A) = \underset{I:I\models A}{\sum} P_{\Pi}(I).
$

\section{$\lpmln$ Weight Learning}\label{sec:lpmln-weight-learning}

\subsection{General Problem Statement} 

A parameterized $\lpmln$ program $\hat{\Pi}$ is defined similarly to an $\lpmln$ program $\Pi$ except that non-$\alpha$ weights (i.e., ``soft" weights) are replaced with distinct parameters to be learned. By~$\hat{\Pi}({\bf w})$, where ${\bf w}$ is a list of real numbers whose length is the same as the number of soft rules, we denote the $\lpmln$ program obtained from~$\hat{\Pi}$ by replacing the parameters  with~${\bf w}$.
The weight learning task for a parameterized $\lpmln$ program is to find the MLE (Maximum likelihood Estimation) of the parameters as in Markov Logic.
Formally, given a parameterized $\lpmln$ program $\hat{\Pi}$ and a ground formula ${O}$ (often in the form of conjunctions of literals) called {\em observation} or {\em training data}, the $\lpmln$ parameter learning task is to find the values ${\bf w}$ of parameters such that the probability of ${O}$ under the $\lpmln$ program $\Pi$ is maximized. In other words, the learning task is to find
\beq
  \underset{{\bf w}}{\rm argmax}\ P_{\hat{\Pi}({\bf w})}(O). 
\eeq{learning}

\subsection{Gradient Method for Learning Weights From a Complete Stable Model}\label{sec:gradient-descent}


Same as in Markov Logic, there is no closed form solution for \eqref{learning} but the gradient ascent method can be applied to find the optimal weights in an iterative manner.

We first compute the gradient. Given a (non-ground) $\lpmln$ program $\Pi$ whose $\sm[\Pi]$ is non-empty and given a stable model $I$ of $\Pi$, 
%
the base-$e$ logarithm of $P_{\Pi}(I)$, $ln P_{\Pi}(I)$, is
\NBB{
{\cred [$P_{\Pi}(I)$ looks a bit misleading. It looks like the gradient of $P_{\Pi}(I)$]}
}
{\small
\begin{align*}
   -\underset{w_i: R_i\in \Pi^{\rm soft}}{\sum}w_in_i(I) - ln\!\!\! \underset{J\in \sm[\Pi]}{\sum}exp\Big(-\underset{w_i: R_i\in \Pi^{\rm soft}}{\sum}w_in_i(J)\Big).
\end{align*} 
}
The partial derivative of $ln P_{\Pi}(I) $ w.r.t. $w_i (\ne \alpha)$ is
{\tiny 
\begin{align}
\nonumber & \frac{\partial ln P_{\Pi}(I)}{\partial w_i} =
     -n_i(I) + \frac{\underset{J\in \sm[\Pi]}{\sum}exp(-\underset{w_i: R_i\in \Pi^{\rm soft}}{\sum}w_in_i(J))n_i(J)}{\underset{K\in \sm[\Pi]}{\sum}exp(-\underset{w_i: R_i\in \Pi^{\rm soft}}{\sum}w_in_i(K))}\\
\nonumber &= -n_i(I) + \underset{J\in \sm[\Pi]}{\sum}
   \Bigg(\frac{exp(-\underset{w_i: R_i\in \Pi^{\rm soft}}{\sum}w_in_i(J))}
            {\underset{K\in\sm[\Pi]}{\sum}exp(-\underset{w_i: R_i\in \Pi^{\rm soft}}{\sum}w_in_i(K))}\Bigg)
   n_i(J)\\
\nonumber &= -n_i(I) + \underset{J\in \sm[\Pi]}{\sum}P_{\Pi}(J)n_i(J) = -n_i(I) + \underset{J\in\sm[\Pi]}E[n_i(J)] 
\end{align} 
}

\noindent
where $\underset{J\in\sm[\Pi]}E[n_i(J)]= \underset{J\in \sm[\Pi]}{\sum}P_{\Pi}(J)n_i(J)$ is
the expected number of false ground rules obtained from $R_i$.

Since the log-likelihood above is a concave function of the weights, any local maximum is a global maximum, and maximizing $P_\Pi(I)$ can be done by the standard gradient ascent method  by updating each weight $w_i$ by 
$w_i+\lambda\cdot (-n_i(I) + \underset{J\in \sm[\Pi]}{E}[n_i(J)])$ until it converges.\footnote{Note that although any local maximum is a global maximum for the log-likelihood function, there can be multiple combinations of weights that achieve the maximum probability of the training data.} 

However, similar to Markov Logic, computing $\underset{J\in \sm[\Pi]}{E}[n_i(J)]$ is intractable \cite{richardson06markov}.
In the next section, we turn to an MCMC sampling method to find its approximate value.


\subsection{Sampling Method: MC-ASP}

The following is an MCMC algorithm for $\lpmln$, which adapts the algorithm MC-SAT for Markov Logic \cite{poon06sound} by considering the penalty-based reformulation and by using an ASP solver instead of a SAT solver for sampling.

\begin{algorithm}
{\bf Input:} An $\lpmln$ program $\Pi$ whose soft rules' weights are non-positive and a positive integer $N$.

{\bf Output:} Samples $I^{1}, \dots, I^{N}$
\begin{enumerate}
\item Choose a (probabilistic) stable model $I^0$ of $\Pi$.
\item Repeat the following for $j=1,\dots, N$ 
\begin{enumerate}
\item $M\leftarrow \emptyset$;
\item {For each ground instance of each rule $w_i: R_i\in\Pi^{\rm soft}$ that is false in $I^{j-1}$}, add the ground instance to $M$ with probability $1-e^{w_i}$;
\item Randomly choose a (probabilistic) stable model $I^j$ of $\Pi$ that satisfies no rules in $M$.
\end{enumerate}
\end{enumerate}
\caption{MC-ASP}
\label{alg:mc-asp}
\end{algorithm}

When all the weights $w_i$ of soft rules are non-positive, $1-e^{w_i}$ (at step (b)) is in the range $[0, 1)$ and thus it validly represents a probability. At each iteration, the sample is chosen from stable models of $\Pi$, and consequently, it must satisfy all hard rules.
For soft rules, the higher its weight, the less likely that it will be included in $M$, and thus less likely to be not satisfied by the sample generated from $M$.

The following theorem states that MC-ASP satisfies the MCMC criteria of ergodicity and detailed balance, which justifies the soundness of the algorithm.

\begin{thm}\label{thm:mc-sat-lpmln-pnt-correctness}
The Markov chain generated by MC-ASP satisfies ergodicity and detailed balance.\footnote{
A Markov chain is {\em ergodic} if there is a number $m$ such that any state can be reached from any other state in any number of steps greater than or equal to $m$.

{\em Detailed balance} means $P_{\Pi}(X)Q(X\rightarrow Y) = P_{\Pi}(Y)Q(Y\rightarrow X)$ for any samples $X$ and $Y$, where $Q(X\rightarrow Y)$ denotes the probability that the next sample is $Y$ given that the current sample is $X$. 
}
\end{thm}

Steps 1 and 2(c) of the algorithm require finding a probabilistic stable model of $\lpmln$, which can be computed by system {\sc lpmln2asp} \cite{lee17computing}. The system is based on the translation that turns an $\lpmln$ program $\Pi$ into an ASP program ${\sf lpmln2asp}^{\rm pnt}(\Pi)$. The translation turns each (possibly non-ground) soft rule
\beq
w_i:\ \ \ \i{Head}_i({\bf x}) \leftarrow \i{Body}_i({\bf x})
\eeq{rule}
into~\footnote{If $\i{Head}_i({\bf x})$ is a disjunction of atoms $a_1({\bf x})\ ;\ \dots\ ;\ a_n({\bf x})$, then ${\tt not}\ \i{Head}_i({\bf x})$ denotes ${\tt not}\ a_1({\bf x}), \dots, {\tt not}\ a_n({\bf x})$.}
\begin{align}
\nonumber {\tt unsat}(i, w_i, {\bf x}) \leftarrow \i{Body}_i({\bf x}),\ {\tt not}\ \i{Head}_i({\bf x}) \\
\nonumber \i{Head}_i({\bf x}) \leftarrow  \i{Body}_i({\bf x}),\ {\tt not}\ {\tt unsat}(i, w_i, {\bf x})\\
\nonumber :\sim {\tt unsat}(i, w_i, {\bf x}). \ \ [w_i, i, {\bf x}]
\end{align}
and each hard rule 
\[ 
\alpha: \ \ \ \i{Head}_i({\bf x}) \leftarrow \i{Body}_i({\bf x})
\]
into $\i{Head}_i({\bf x}) \leftarrow \i{Body}_i({\bf x})$.
System {\sc lpmln2asp} turns an $\lpmln$ program $\Pi$ into ${\sf lpmln2asp}^{\rm pnt}(\Pi)$ and calls ASP solver {\sc clingo} to find the stable models of ${\sf lpmln2asp}^{\rm pnt}(\Pi)$, which coincide with the probabilistic stable models of $\Pi$. The weight of a stable model can be computed from the weights recorded in ${\tt unsat}$ atoms that are true in the stable model.

Step 2(c) also requires a uniform sampler for answer sets, which can be computed by {\sc xorro} \cite{gebser16xorro}.

\begin{algorithm}[h!]
{\footnotesize
\noindent {\bf Input: }
$\Pi$: A parameterized $\lpmln$ program in the input language of {\sc lpmln2asp};
$O$: A stable model represented as a set of constraints (that is, 
$\ar \no\ A$ is in $O$ if a ground atom $A$ is true;  $\ar A$ is in $O$ if $A$ is not true);
$\delta$: a fixed real number to be used for the terminating condition.

\noindent {\bf Output: } $\Pi$ with learned weights.

\noindent {\bf Process:}
\begin{enumerate}
\item Initialize the weights of soft rules $R_1, \dots, R_m$ with some initial weights ${\bf w}^0$.
\item Repeat the following for $j=1,\dots$ until \mbox{${max}\{|w_i^j - w_i^{j-1}| : i=1, \dots, m\}<\delta$}:
\begin{enumerate}
\item Compute the stable model of $\Pi\cup O$ using {\sc lpmln2asp} (see below); for each soft rule $R_i$, compute $n_i(O)$ by counting ${\tt unsat}$ atoms whose first argument is $i$ ($i$ is a rule index).
\item Create $\Pi^{neg}$ by replacing each soft rule $R_i$ of the form $w:\ H({\bf x})\leftarrow B({\bf x})$ in $\Pi$ where $w>0$ with
\[
\ba l
0:\ H({\bf x})\leftarrow B({\bf x}), \\
\alpha:\ {\tt neg}(i, {\bf x}) \leftarrow B({\bf x}), {\tt not}\ H({\bf x}), \\
-w:\ \leftarrow {\tt not}\ {\tt neg}(i, {\bf x}).
\ea
\]

\item Run MC-ASP on $\Pi^{neg}$ to collect a set $S$ of sample stable models.

\item For each {soft rule $R_i$}, approximate $\underset{J\in\sm[\Pi]}{\sum}P_{\Pi}(J)n_i(J)$ with ${\underset{J\in S}{\sum}n_i(J)}/{|S|}$, where $n_i$ is obtained from counting the number of ${\tt unsat}$ atoms whose first argument is $i$. 


\item For each $i\in \{1, \dots, m\}$, \\ $w_i^{j+1}  \leftarrow w_i^{j} + \lambda\cdot (-n_i(O)+{\underset{J\in S}{\sum}n_i(J)}/{|S|})$.

\end{enumerate}
\end{enumerate}
\caption{Algorithm for learning weights using {\sc lpmln2asp}}
\label{alg:learning}
}
\end{algorithm}

Algorithm~\ref{alg:learning} is a weight learning algorithm for $\lpmln$ based on gradient ascent using MC-ASP ({Algorithm}~\ref{alg:mc-asp}) for collecting samples.
Step 2(b) of MC-ASP requires that $w_i$ be non-positive in order for $1-e^{w_i}$ to represent a probability. Unlike in the Markov Logic setting, converting positive weights into non-positive weights cannot be done  in $\lpmln$ simply by replacing $w:F$ with $-w:\neg F$, 
due to the difference in the FOL and the stable model semantics. 
Algorithm~\ref{alg:learning} converts $\Pi$ into an equivalent program $\Pi^{neg}$ whose rules' weights are non-positive, before calling MC-ASP. 
The following theorem justifies the soundness of this 
method.\footnote{Note that $\Pi^{neg}$ is only used in MC-ASP. The output of Algorithm~\ref{alg:learning} may have positive weights.}

\begin{thm}\label{thm:pi-neg-equivalence}
When $\sm[\Pi]$ is not empty, the program $\Pi^{neg}$ specifies the same probability distribution as the program~$\Pi$.\footnote{Non-emptiness of $\sm[\Pi]$ implies that every probabilistic stable model of $\Pi$ satisfies all hard rules in~$\Pi$.}
\end{thm}

\section{Extensions} \label{sec:extensions}
The base case learning in the previous section assumes that the training data is a single stable model {and is a complete interpretation}. This section extends the framework in a few ways.
\subsection{Learning from Multiple Stable Models} \label{sec:multiple-evidence}
The method described in the previous section 
allows only one stable model to be used as the training data. Now, suppose we have multiple stable models $I_1, \dots, I_m$ as the training data. For example, consider the parameterized program $\hat{\Pi}_{coin}$ that describes a coin, which may or may not land in the head when it is flipped,
\begin{align}
\nonumber \alpha\ \ :&\ \ \{flip\} \\
\nonumber w\ \ :&\ \ head \leftarrow flip
\end{align}
(the first rule is a choice rule) 
and three stable models as the training data: $I_1=\{flip\}$, $I_2=\{flip\}$, $I_3=\{flip, head\}$ (the absence of $head$ in the answer set is understood as landing in tail), indicating that 
$\{flip, head\}$ has a frequency of $\frac{1}{3}$, and 
$\{flip\}$ has a frequency of $\frac{2}{3}$. 
Intuitively, the more we observe the $head$, the larger the weight of the second rule. 
Clearly, learning $w$ from only one of $I_1, I_2, I_3$ won't result in a weight that captures all the three stable models: learning from each of $I_1$ or $I_2$ results in the value of $w$ too small for $\{flip, head\}$ to have a frequency of $\frac{1}{3}$ while learning from $I_3$ results  in the value of $w$ too large for $\{flip\}$ to have a frequency of $\frac{2}{3}$.

To utilize the information from multiple stable models, one natural idea is to maximize the joint probability of all the stable models in the training data, which is the product of their probabilities, i.e., 
\[
  P(I_1, \dots, I_m) = \underset{j\in\{1,\dots,m\}}{\prod}P_{\Pi}(I_j).
\]

The partial derivative of $ln P(I_1, \dots, I_m) $ w.r.t. $w_i (\ne \alpha)$ is
{\footnotesize
\begin{align*}
 \frac{\partial ln P(I_1, \dots, I_m)}{\partial w_i}
%
=
\underset{j\in\{1,\dots,m\}}{\sum}\Big(-n_i(I_j) + \underset{J\in\sm[\Pi]}E[n_i(J)]\Big).
\end{align*} 
}

\noindent
In other words, the gradient of the log probability is simply the sum of the gradients of the probability of each stable model in the training data. To update {Algorithm}~\ref{alg:learning} to reflect this, we simply repeat step 2(a) to compute ${n_i}(I_k)$ for each $k\in\{1, \dots, m\}$, and at step 2(e) update $w_i$ as follows:
{\footnotesize
\[
w_i^{j+1}  \leftarrow w_i^{j} + \lambda\cdot\Big(-\underset{k\in\{1, \dots, m\}}{\sum}n_i(I_k)+m\cdot\underset{J\in\sm[\Pi]}{\sum}P_{\Pi}(J)n_i(J)\Big).
\]
}

Alternatively, learning from multiple stable models can be reduced to learning from a single stable model by introducing one more argument $k$ to every predicate, which represents the index of a stable model in the training data, and rewriting the data to include the index. 
\BOCCC
In this way, since there are more ground instances of the rule whose weight is to be learned, in our equation
\[
\frac{\partial ln P_{\Pi}(I)}{\partial w_i}=-n_i(I) + \underset{J\in\sm[\Pi]}E[n_i(J)],
\]
$n_i(I)$ ranges over a larger range of integers, making it more expressive in capturing how much the rule is satisfied.
\EOCCC


Formally, given an $\lpmln$ program $\Pi$ and a set of its stable models $I_1,\dots, I_m$, let $\Pi^m$ be an $\lpmln$ program obtained from $\Pi$ by appending one more argument $k$ to the list of arguments of every predicate that occurs in $\Pi$, where $k$ is a schematic variable that ranges over $\{1, \dots, m\}$. Let
\beq
   I  = \bigcup_{i\in\{1,\dots,m\}} \{p({\bf t}, i)
   \mid p({\bf t})\in {I}_i\}.
\eeq{multiple-sm}

The following theorem asserts that the weights of the rules in $\Pi$ that are learned from the multiple stable models $I_1, \dots, I_m$ are identical to the weights of the rules in $\Pi^m$ that are learned from the single stable model $I$ that conjoins $\{I_1,\dots, I_m\}$ as in~\eqref{multiple-sm}.


\begin{thm}\label{thm:multiple-evidence}
For any parameterized $\lpmln$ program $\hat{\Pi}$, its stable models $I_1,\dots, I_m$ and $I$ as defined as in~\eqref{multiple-sm}, we have
\[
\underset{{\bf w}}{\rm argmax}\ P_{{\hat{\Pi}^m}({\bf w})}(I)  =
\underset{{\bf w}}{\rm argmax}\ 
    \underset{i\in\{1, \dots, m\}}{\prod}P_{\hat{\Pi}({\bf w})}(I_i).
\]
\end{thm}

\begin{example}\label{ex:coin}
For the program $\hat{\Pi}_{coin}$, to learn from the three stable models $I_1$, $I_2$, and $I_3$ defined before, we consider the program $\hat{\Pi}_{coin}^3$ 
\begin{align}
\nonumber \alpha\ \ :&\ \ \{flip(k)\}.\\
\nonumber w\ \ :&\ \ head(k) \leftarrow flip(k).
\end{align}
($k\in\{1,2,3\}$)
and combine $I_1, I_2, I_3$ into one stable model $I=\{flip(1), flip(2), flip(3), head(3)\}$. 
The weight $w$ in $\hat{\Pi}_{coin}^3$ learned from the single data $I$ is identical to the weight $w$ in $\hat{\Pi}_{coin}$ learned from the three stable models $I_1,I_2,I_3$.
\end{example}

\subsection{Learning in the Presence of Noisy Data} 

\NBB{This section needs to be rewritten; it does not conform to the assumption} 

So far, we assumed that the data $I_1,\dots, I_m$ are (probabilistic) stable models of the parameterized $\lpmln$ program. 
Otherwise, the joint probability would be zero regardless of any weights assigned to the soft rules, and the partial derivative of $ln P(I_1,\dots, I_m)$ is undefined. 
However, data gathered from the real world could be noisy, so some data $I_i$ may not necessarily be a stable model. Even then, we still want to learn from the other ``correct" instances. We may drop them in the pre-processing to learning but this could be computationally expensive if the data is huge. Alternatively, we may mitigate the influence of the noisy data by introducing so-called ``noise atoms'' as follows.
\begin{example}
Consider again the program $\hat{\Pi}_{coin}^m$. 
Suppose one of the interpretations $I_i$ in the training data is $\{head(i)\}$.  The interpretation is not a stable model of $\hat{\Pi}_{coin}^m$. 
We obtain $\hat{\Pi}^m_{noisecoin}$ by modifying  $\hat{\Pi}_{coin}^m$ to allow for the noisy atom $n(k)$ as follows.
{
\begin{align}
\nonumber \alpha\ \ :&\ \ \{flip(k)\}.\\
\nonumber w\ \ :&\ \ head(k) \leftarrow flip(k).\\
\nonumber \alpha \ \ :&\ \ head(k) \leftarrow n(k).\\
\nonumber -u\ \ :&\ \ n(k).
\end{align}
}

Here, $u$ is a positive number that is ``sufficiently'' larger than~$w$. 
$\{head(i), n(i)\}$ is a stable model of $\hat{\Pi}^m_{noisecoin}$, so that the combined training data $I$ is still a stable model, and thus a meaningful weight $w$ for $\hat{\Pi}^m_{noisecoin}$ can still be learned, given that other ``correct'' instances $I_j$ ($j\ne i$) dominate in the learning process (as for the noisy example, the corresponding stable model gets a low weight due to the weight assigned to $n(i)$ but not 0).

Furthermore, with the same value of $w$, the larger $u$ becomes, the closer the probability distribution defined by $\hat{\Pi}^m_{noisecoin}$ approximates the one defined by $\hat{\Pi}_{coin}^m$, so the value of $w$ learned under $\hat{\Pi}^m_{noisecoin}$ approximates the value of $w$ learned under $\hat{\Pi}_{coin}^m$ where the noisy data is dropped.
\end{example}

\NBB{This is as if throw away $I_i$? but withoutknowing which is noisy beforehand, we should use our method, or can we do filtering beforehand? }

\subsection{Learning from Incomplete Interpretations}\label{ssec:partial}

In the previous sections, we assume that the training data is given as a (complete) interpretation, i.e., for each atom it specifies whether it is true or false. In this section, we discuss the general case when the training data is given as a partial interpretation, which omits to specify some atoms to be true or false, or more generally when the training data is in the form of a formula that more than one stable model may satisfy. 

Given a non-ground $\lpmln$ program $\Pi$ such that $\sm[\Pi]$ is not empty and given a ground formula $O$ as the training data, we have
\[
  P_\Pi(O) = \frac{ \sum_{I\models O, I\in\sm[\Pi]} W_\Pi(I)}
      {\sum_{J\in{\rm SM}[\Pi]}{W_{\Pi}(J)}}.
\] 

The partial derivative of $ln P_\Pi(O)$ w.r.t. $w_i$ ($\ne \alpha$) turns out to be
\[
\frac{\partial ln P_{\Pi}(O)}{\partial w_i} = -\underset{I\models O, I\in \sm[\Pi]}{E}[n_i(I)] + \underset{J\in \sm[\Pi]}{E}[n_i(J)].
\]

\BOC
A gradient descent based learning algorithm for partial interpretation as evidence is as follows:

{\bf Input: $\Pi$(program), $O$(Evidence), $\delta$(Terminating condition), $\lambda$(Learning Rate)}
\begin{enumerate}
\item Initialize the weights of $U=\{R_1, \dots, R_m\}$ with some ${\bf w}^0=\{w_1^0, \dots, w_m^0\}$;
\item Repeat the follows until $\underset{j=1,\dots, m}{max}\{|w_j^i - w_j^{i-1}|\}<\delta$:
\begin{enumerate}
\item For each $j\in\{1, \dots, m\}$, compute an approximation of $\underset{J\in \sm[\Pi]}{E}[n_j(J)]$, denoted by $\tilde{E}(n_j)$;
\item For each $j\in\{1, \dots, m\}$, compute an approximation of $\underset{I\models O, I\in \sm[\Pi]}{E}[n_j(I)]$, denoted by $\tilde{E_O}(n_j)$;
\item For each $j\in\{1, \dots, m\}$, $w_j^{i+1} \leftarrow w_j^i+\lambda\cdot (-\tilde{E_O}(n_j)+\tilde{E}(n_j))$
\end{enumerate}
\end{enumerate}
\EOC
It is straightforward to extend {Algorithm}~\ref{alg:learning} to reflect the extension. 
Computing the approximate value of the first term 
$-\underset{I\models O, I\in \sm[\Pi]}{E}[n_i(I)]$ can be done by sampling on $\Pi^{neg}\cup O$.


\section{$\lpmln$ Weight Learning via Translations to Other Languages}\label{sec:learning-other}

This section considers two fragments of $\lpmln$, for which the parameter learning task reduces to the same tasks for Markov Logic and  ProbLog.

\subsection{Tight $\lpmln$ Program: Reduction to MLN Weight Learning}

By Theorem 3 in \cite{lee16weighted}, any tight $\lpmln$ program can be translated into a Markov Logic Network (MLN) by adding completion formulas \cite{erd03} with the weight $\alpha$. This means that the weight learning for a tight $\lpmln$ program can be reduced to the weight learning for an MLN. 

Given a tight $\lpmln$ program $\Pi = \langle {\bf R}, {\bf W}\rangle$ and one (not necessarily complete) interpretation $E$ as the training data, the MLN $Comp(\Pi)$ is obtained by adding completion formulas  with weight $\alpha$ to $\Pi$.

The following theorem 
tells us that the weight assignment that maximizes the probability of the training data under $\lpmln$ programs is identical to the weight assignment that maximizes the probability of the same training data under an MLN $Comp(\Pi)$. 

\begin{thm}\label{thm:lpmln-learn-to-mln-learn}
Let ${\rm L}$ be the Markov Logic Network $Comp(\Pi)$ and let $E$ be a ground formula (as the training data).
When $\sm[\Pi]$ is not empty, 
\[
\underset{{\bf {\bf w}}}{\rm argmax}\ P_{\hat{\Pi}({\bf w})}(E) = 
\underset{{\bf w}}{\rm argmax}\ P_{\hat{\rm L}({\bf w})}(E). 
\]
($\hat{\rm L}$ is a parameterized MLN obtained from ${\rm L}$.) 
\end{thm}

Thus we may learn the weights of a tight $\lpmln$ program using the existing implementations of Markov Logic, such as {\sc alchemy} and {\sc tuffy}.
\subsection{Coherent $\lpmln$ Program: Reduction to Parameter Learning in ProbLog}

\NBB{needs to make it compact}

For another special class of $\lpmln$ programs, {weight} learning can be reduced to {weight} learning in ProbLog \cite{fierens13inference}.

We say an $\lpmln$ program $\Pi$ is {\em simple} if all soft rules in $\Pi$ are of the form
\[
   w: A
\]
where $A$ is an atom, and no atoms occurring in the soft rules occur in the head of a hard rule.


We say a simple $\lpmln$ program $\Pi$ is {\em $k$-coherent} ($k>0$) if, for any truth assignment to atoms that occur in $\Pi^{\rm soft}$, there are exactly $k$ probabilistic stable models of $\Pi$ that satisfies the truth assignment. We also apply the notion of $k$-coherency when $\Pi$ is parameterized.

Without loss of generality, we assume that no atom occurs more than once in $\Pi^{\rm soft}$. (If one atom $A$ occurs in multiple rules $w_1: A, \dots, w_n:A$, these rules can be combined into $w_1+\dots+w_n:A$.) A $k$-coherent $\lpmln$ program $\Pi$ can thus be identified with the tuple $\langle PF, \Pi^{\rm hard}, {\bf w}\rangle$, where $PF=(pf_1, \dots, pf_m)$ is a list of (possibly non-ground) atoms that occur as soft rules in $\Pi$, $\Pi^{\rm hard}$ is a set of hard rules in $\Pi$, and ${\bf w} = (w_1, \dots, w_m)$ is the list of soft rule's weights, where $w_i$ is the weight of $pf_i$.


A ProbLog program can be viewed as a tuple $\langle PF, {\bf R}, {\bf pr}\rangle$ where $PF$ is a list of atoms called {\em probabilistic facts}, ${\bf R}$ is a set of rules such that no atom that occurs in $PF$ occurs in the head of any rule in ${\bf R}$, and ${\bf pr}$ is a list $(p_1, \dots, p_{|PF|})$, where each $p_i$ is the probability of probabilistic atom $pf_i\in PF$ . A {\em parameterized} ProbLog program is similarly defined, where ${\bf pr}$ is a list of parameters to be learned. 


Given a list of probabilities ${\bf pr}=(p_1, \dots, p_n)$, we construct a list of weights ${\bf w}^{\bf pr}=(w_1, \dots, w_n)$ as follows:
\begin{equation}\label{eq:probability2weight}
w_i = ln(\frac{p_i}{1-p_i}) 
\end{equation}
for $i\in \{1, \dots n\}$. 

The following theorem asserts that
weight learning on a 1-coherent $\lpmln$ program can be done by weight learning on its corresponding ProbLog program. 


\begin{thm}\label{cor:lpmln-learn2problog-learn}
For any 1-coherent parameterized $\lpmln$ program $\langle PF, P, {\bf w}\rangle$ and any interpretation $T$ (as the training data), we have 
\begin{align}
\nonumber &{\bf w} = \underset{{\bf w}}{\rm argmax}\ P_{\langle PF, P, {\bf w}\rangle}(T)\\
\nonumber &\text{if and only if}\\
\nonumber &{\bf w} = {\bf w}^{\bf pr}\text{ and } {\bf pr} = \underset{{\bf pr}}{\rm argmax}\ P_{\langle PF, P, {\bf pr}\rangle}(T).
\end{align}
\end{thm}


According to the theorem, 
to learn the weights of a 1-coherent $\lpmln$ program, we can simply construct the corresponding ProbLog program, perform ProbLog weight learning, and then turn the learned probabilities into $\lpmln$ weights according to \eqref{eq:probability2weight}.



In~\cite{lee18aprobabilistic}, $k$-coherent programs are shown to be useful for describing dynamic domains. Intuitively, each probabilistic choice leads to the same number of histories. For such a $k$-coherent $\lpmln$ program, weight learning given a complete interpretation as the training data can be done by simply counting true and false ground instances of soft atomic facts in the given interpretation.

For an interpretation $I$ and $c_i\in PF$, let $m_i(I)$ and $n_i(I)$ be the numbers of ground instances of $c_i$ that is true in $I$ and false in $I$, respectively.

\begin{thm}\label{cor:coherent-learning}
For any $k$-coherent parameterized $\lpmln$ program $\langle PF, \Pi^{\rm hard}, {\bf w}\rangle$, and any (complete) interpretation $I$ (as  the training data), we have
{\small 
\[
\underset{\bf w}{\rm argmax}\ P_{\langle PF, \Pi^{\rm hard}, {\bf w}\rangle}(I; {\bf w}) =\Big(ln\frac{m_1(I)}{n_1(I)}, \dots, ln\frac{m_{|PF|}(I)}{n_{|PF|}(I)}\Big).
\]
}
\end{thm}

\NBB{
[[Cohrent pr programs generalize ProbLog; it's useful for BC+. The folwing theroe is essentially a agnealization from ProblLog, wich hints that Problog algorithm canb e adapted to ]] }

\BOC
{In \cite{gutmann2011learning}, the authors had similar observation (see equation (2) in \cite{gutmann2011learning}). It tells us that weight learning for $k$-coherent $\lpmln$ program given a complete interpretation as the training data can be done much more efficiently than in the general setting.}
\EOC

\section{Implementation and Examples}\label{sec:examples}
We implemented Algorithm~\ref{alg:learning} and its extensions described above using {\sc clingo}, {\sc lpmln2asp}, and a near-uniform answer set sampler {\sc xorro
}. 
The implementation {\sc lpmln-learn} is available at \url{https://github.com/ywng485/lpmln-learning} together with a manual and some examples.

\NBB{ current imlementation is prototype; sampler is not good }



In this section, we show how the implementation allows for learning weights in $\lpmln$ from the data enabling learning parameters in knowledge-rich domains. 

For all the experiments in this section, $\delta$ is set to be $0.001$. $\lambda$ is fixed to $0.1$ and $50$ samples are generated for each call of MC-ASP. The parameters for {\sc xorro} are manually tuned to achieve the best performance for each specific example.



\subsection{Learning Certainty Degrees of Hypotheses}\label{ssec:learn-hypothesis}

The $\lpmln$ weight learning algorithm can be used to learn the certainty degree of a hypothesis from the data. For example, consider a person $A$ carrying a certain virus contacting a group of people. The virus spreads among them as people contact each other. 
We use the following ASP facts to specify that $A$ carries the virus and how people contacted each other:
\begin{lstlisting}
  carries_virus("A").
  contact("A", "B").  contact("B", "C").  ...
\end{lstlisting}
Consider two hypotheses that a person carrying the virus may cause him to have a certain disease, and the virus may spread by contact. 
The hypotheses can be represented in the input language of {\sc lpmln-learn} by the following rules, where {\tt \@w(1)} and {\tt \@w(2)} are parameters to be learned:
\begin{lstlisting}
@w(1) has_disease(X) :- carries_virus(X).
@w(2) carries_virus(Y) :- contact(X, Y), 
                          carries_virus(X).
\end{lstlisting}
The parameterized $\lpmln$ program consists of these two rules and the facts about  \lstinline{contact} relation. The training data specifies whether each person carries the virus and has the disease, for example:
\begin{lstlisting}
:- not carries_virus("E").   :- carries_virus("H").
...
:- not has_disease("A").      :- has_disease("H").
\end{lstlisting}

The learned weights tell us how certain the data support the hypotheses. 
Note that the program models the transitive closure of the \lstinline{carries_virus} relation, which is not properly done if the program is viewed as an MLN.\footnote{That is, identifying the rule $H\ar B$ with a formula in first-order logic $B\rar H$.} Learning under the MLN semantics results in weights that associate unreasonably high probabilities to people carrying virus even if they were not contacted by people with virus.

For example, consider the following graph 
\begin{center}
\includegraphics[width=0.5\linewidth]{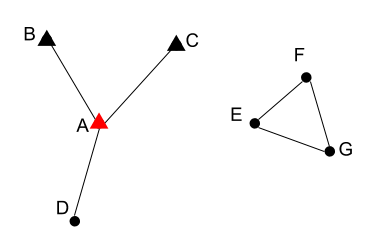}
\end{center}
where {\tt A} is the person who initially carries the virus,  triangle-shaped nodes represent people who carry virus in the evidence, and the edges denote the {\tt contact} relation. The cluster consisting of {\tt E}, {\tt F}, and {\tt G} has no contact with the cluster consisting of {\tt A}, {\tt B}, {\tt C}, and {\tt D}.
The following table shows the probability of each person carrying the virus, which is derived from the weights learned in accordance with Markov Logic and $\lpmln$, respectively. We use {\sc alchemy} for the weight learning in Markov Logic.

\medskip
{\footnotesize
\centering
\begin{tabular}{| c | c c c|}
\hline
 {\bf Person} & {\bf {\sc MLN}} & {\bf $\lpmln$} & carries\_virus \\  
       &   &   & (ground truth) \\ \hline
 $B$ & 0.823968 & 0.6226904833 & Y \\
 $C$ & 0.813969 & 0.6226904833 & Y \\
 $D$ & 0.818968 & 0.6226904833 & N \\
 $E$ & 0.688981 & 0 & N \\
$F$ & 0.680982 & 0 & N \\
$G$ & 0.680982 & 0 & N \\
 \hline
\end{tabular}
}

\medskip
As can be seen from the table, under MLN, each of {\tt E}, {\tt F}, {\tt G} has a high probability of carrying the virus, which is unintuitive. 



\subsection{Learning Probabilistic Graphs from Reachability} \label{ssec:learn-reachability}
Consider an (unstable) communication network such as the one in Figure~\ref{fig:com-net}, where each node represents a signal station that sends and receives signals. A station may fail, making it impossible for signals to go through the station. The following $\lpmln$ rules define the connectivity between two stations {\tt X} and {\tt Y} in session {\tt T}.

\begin{lstlisting}
connected(X,Y,T) :- edge(X,Y), not fail(X,T), 
                               not fail(Y,T).
connected(X,Y,T) :- connected(X,Z,T), connected(Z,Y,T).
\end{lstlisting}
A specific network can be defined by specifying edge relations, such as \lstinline{edge(1,2)}.
Suppose we have data showing the connectivity between stations in several sessions. Based on the data, we could make decisions such as which path is most reliable to send a signal between the two stations. Under the $\lpmln$ framework, this can be done by learning the weights representing the failure rate of each station. For the network in Figure \ref{fig:com-net}, we write the following rules whose weights ${\tt w}(i)$ are to be learned:
\begin{lstlisting}
@w(1) fail(1, T).    ...     @w(10) fail(10, T).
\end{lstlisting}

\begin{figure}[ht]
\begin{center}
\includegraphics[width=4cm]{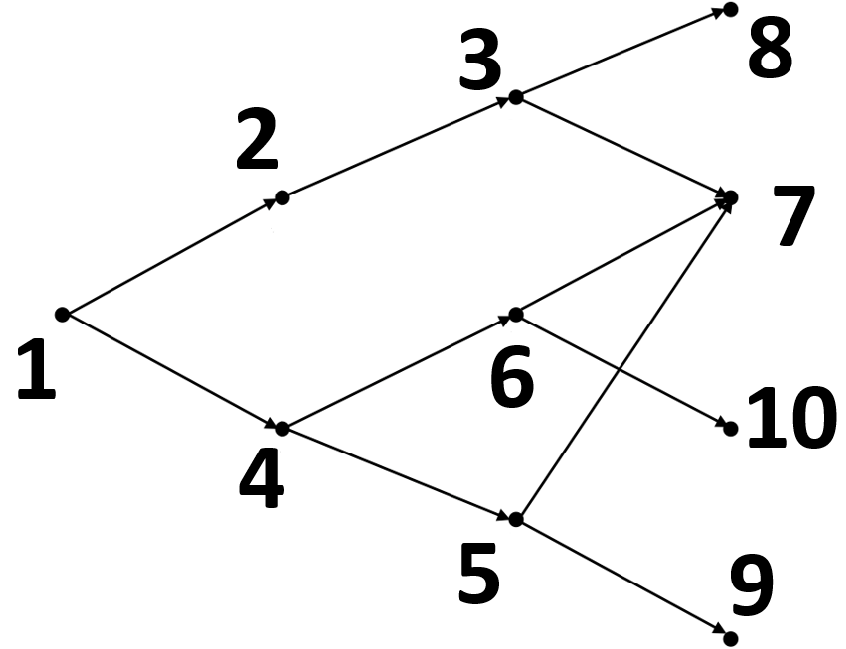} 
\end{center}
\caption{Example Communication Network}
\label{fig:com-net}
\end{figure}
Here ${\tt T}$ is the auxiliary argument to allow learning from multiple training examples, as described in Section~\ref{sec:multiple-evidence}.
The training example contains constraints either \lstinline{:- not connected(X,Y)} for known connected stations \lstinline{X} and \lstinline{Y} or \lstinline{:- connected(X,Y)} for known disconnected stations \lstinline{X} and \lstinline{Y}.
Since the training data is incomplete in specifying the connectivity between the stations, 
we use the extension of Algorithm~\ref{alg:learning} described in Section~\ref{ssec:partial}. 
The failure rates of the stations can be obtained from the learned weights as 
$\frac{e^{{\tt w}(i)}}{{e^0 + e^{{\tt w}(i)}}}$.

\NBB{{\cred \LARGE [[The probabilities are higher than the theoretical values because the samples returned from Flavio's component are biased towards less failed nodes]]}}


%

\begin{figure}[ht]
\begin{center}
\includegraphics[width=1.0\linewidth, height=5cm]{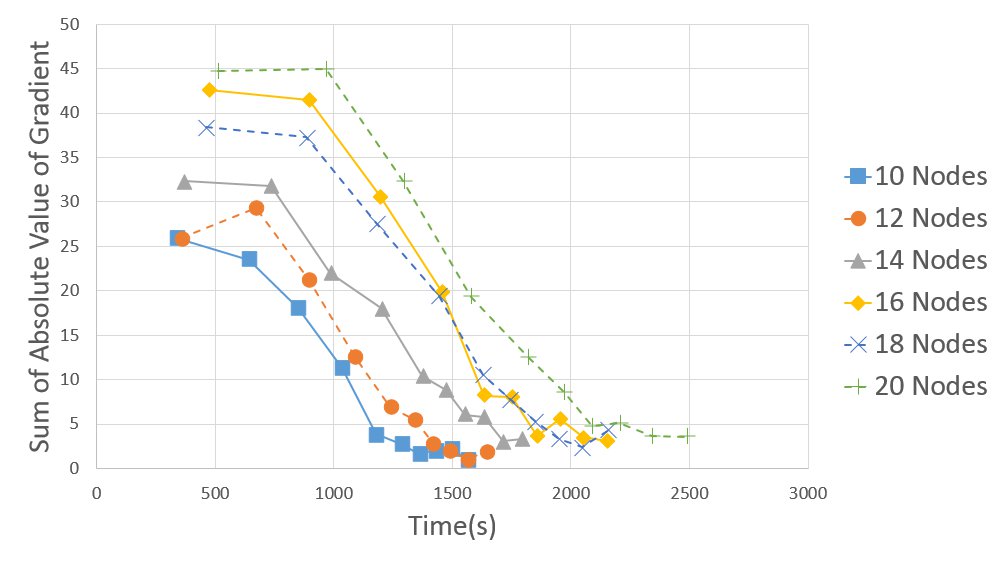} 
\end{center}
\caption{Convergence Behavior of Failure Rate Learning}
\label{fig:com-net-convergence-sum-gradient}
\end{figure}

We execute learning on graphs with $10, 12, \dots, 18, 20$ nodes, where the graph with $10$ nodes is shown in Figure~\ref{fig:com-net}. We add $1, 2, \dots, 5$ layers of 2 nodes between Node $1$ and Node $2,4$ to obtain the other graphs, where there is an edge between every node in one layer and every node in the previous and next layer. Figure \ref{fig:com-net-convergence-sum-gradient} shows the convergence behavior over time in terms of the sum of the absolute values of gradients of all weights. Running time is mostly spent by the uniform sampler for answer sets. The experiments are performed on a machine with 4 Intel(R) Core(TM) i5-2400 CPU with OS Ubuntu 14.04.5 LTS and 8 GB memory.


Figure~\ref{fig:com-net-convergence-sum-gradient} shows that convergence takes longer as the number of nodes increases, which is not surprising. 
Note that the current implementation is not very efficient. Even for graphs with $10-20$ nodes, it takes $1500-2000$ seconds to obtain a reasonable convergence. The computation bottleneck lies in the uniform sampler used in Step 2(c) of Algorithm~\ref{alg:mc-asp} whereas creating $\Pi_{neg}$ and turning $\lpmln$ programs into ASP programs are done instantly.
The uniform sampler that we use, {\sc xorro}, follows Algorithm 2 in \cite{gomes07near}. It uses a fixed number of random XOR constraints to prune out a subset of stable models, and randomly select one remaining stable model to return. The process of solving for all stable models after applying XOR constraints can be very time-consuming.

% ...


In this example, it is essential that the samples are generated by an ASP solver because information about node failing needs to be correctly derived from the connectivity, which involves reasoning about the transitive closure.

As Theorem \ref{cor:lpmln-learn2problog-learn} indicates, this weight learning task can alternatively be done through ProbLog weight learning. We use
{\sc problog},\footnote{\url{https://dtai.cs.kuleuven.be/problog/}} an implementation of ProbLog.
The performance of {\sc problog} on weight learning depends on the tightness of the input program. We observed that for many tight programs, {\sc problog} appears to have better scalability than our prototype {\sc lpmln-learn}. However, {\sc problog} system does not show a consistent performance on non-tight programs, such as the encoding of the network example above, possibly due to the fact that it has to convert the input program into weighted Boolean formulas, which is expensive for non-tight programs.\footnote{The difference appears to be analogous to the different approaches to handling non-tight programs by answer set solvers, e.g., the translation-based approach such as {\sc assat} and {\sc cmodels} and the native approach such as {\sc clingo}.}  We can identify many graph instances of the network failure example where our prototype system outperforms {\sc problog}, as the density of the graph gets higher. For example, consider the graph in Figure~\ref{fig:com-net}. With the nodes fixed, as we add more edges to make the graph denser, we eventually hit a point when {\sc problog} does not return a result within a reasonable time limit.
Below is the statistics of several instances.
\begin{table}[h!]
{\footnotesize
\begin{tabular}{| c | c c c |}
\hline
 {\bf \# Edges} & {\bf {\sc lpmln-learn}} & {\bf {\sc ProbLog}} & $\substack{\text{{\sc problog}}\\ \text{(with modified program)}}$\\
 \hline
 $10 $ & 351.237s & 2.565s & 0.846s\\
 $14$ & 476.656s & 2.854s & 0.833s \\
 $15$ & 740.656s & $>$ 20 min & 0.957s\\
 $20$ & 484.348s & $>$ 20 min & 76.143s\\
$40$ & 304.407s & $>$ 20 min & 26.642s\\
 \hline
\end{tabular}
}
\end{table}

The input files to {\sc problog} consist of two parts: edge lists and the part that defines the node failure rates and connectivity. The latter is different for the second column and the third column in the table. For the second column it is the same as the input to {\sc lpmln-learn}:
\begin{lstlisting}
t(_)::fail(1).     ...     t(_)::fail(10).

connected(X, Y):- edge(X, Y), not fail(X), not fail(Y).
connected(X, Y):- connected(X, Z), connected(Z, Y).
\end{lstlisting}

For the third column, we rewrite the rules to make the Boolean formula conversion easier for {\sc problog}. The input program is:\footnote{This was suggested by 
Angelika Kimmig (personal communication)}
\begin{lstlisting}
t(_)::fail(1).     ...     t(_)::fail(10).

aux(X, Y) :- edge(X, Y), not fail(X), not fail(Y).
connected(X, Y) :- aux(X, Y).
connected(X, Y) :- connected(X, Z), aux(Z, Y).
\end{lstlisting}

\BOCC
The edge list differs from instance to instance, as an example, the edge list for the instance with 40 edges is
\begin{lstlisting}
edge(1, 2).
edge(1, 4).
edge(2, 3).
edge(4, 5).
edge(4, 6).
edge(3, 7).
edge(6, 7).
edge(5, 7).
edge(3, 8).
edge(6, 10).
edge(5, 9).
edge(1, 3).
edge(1, 5).
edge(1, 6).
edge(1, 7).
edge(1, 8).
edge(1, 9).
edge(1, 10).
edge(2, 4).
edge(2, 5).
edge(2, 6).
edge(2, 7).
edge(2, 8).
edge(2, 9).
edge(2, 10).
edge(3, 4).
edge(3, 5).
edge(3, 6).
edge(3, 8).
edge(3, 9).
edge(3, 10).
edge(4, 1).
edge(4, 2).
edge(4, 3).
edge(4, 4).
edge(4, 7).
edge(4, 8).
edge(4, 9).
edge(4, 10).
\end{lstlisting}
\EOCC


Although all graph instances have some cycles in the graph, the difference between the instance with 14 edges and 15 edges is the addition of one cycle.
Even with the slight change in the graph, the performance of {\sc problog} becomes significantly slower.


\subsection{Learning Parameters for Abductive Reasoning about Actions}\label{ssec:learn-action}

One of the successful applications of answer set programming is modeling dynamic domains. $\lpmln$ can be used for extending the modeling to allow uncertainty. 
A high-level action language $p{\cal BC}+$ is defined as a shorthand notation for $\lpmln$ \cite{lee18aprobabilistic}. The language allows for probabilistic diagnoses in action domains: given the action description and the histories {where an abnormal behavior occurs}, how to find the reason for the failure? There, the probabilities are specified by the user. This can be enhanced by learning the probability of the failure from the example histories using {\sc lpmln-learn}.\footnote{{ProbLog could not be used in place of $\lpmln$ here because it has the requirement that every total choice leads to exactly one well founded model, and consequently does not support choice rules, which has been used in the formalization of the robot example in this section.}}
In this section, we show how $\lpmln$ weight learning can be used for learning parameters for abductive reasoning in action domains. Due to the self-containment of the paper, instead of showing $p{\cal BC}+$ descriptions, we show its counterpart in $\lpmln$. 

Consider the robot domain described in \cite{iwan02history}: a robot located in a building with 2 rooms {\tt r1} and {\tt r2} and a book that can be picked up. The robot can move to rooms, pick up the book, and put down the book. Sometimes actions may fail: the robot may fail to enter the room, may fail to pick up the book, and may drop the book when it has the book.
The domain can be modeled using answer set programs, e.g., \cite{lif99b}. We illustrate how such a description can be enhanced to allow abnormalities, and how the $\lpmln$ weight learning method can learn the probabilities of the abnormalities given a set of actions and their effects. 

We introduce the predicate 
$Ab(i)$ to represent that some abnormality occurred at step $i$, and the predicate $Ab(AbnormalityName, i)$
to represent that a specific abnormality $AbnormalityName$ occurred at step $i$. The occurrences of specific abnormalities are controlled by probabilistic fact atoms and their preconditions. For example, 
\begin{align}
\nonumber w_1\ &:\ \i{Pf}_1(i)\\
\nonumber \alpha\ &:\  Ab(\i{EnterFailed},i) \leftarrow \i{Pf}_1(i), Ab(i).
\end{align}
defines that the abnormality $\i{EnterFailed}$ occurs with probability { $\frac{e^{w_1}}{e^{w_1}+1}$} {(controlled by the weighted atomic fact $\i{Pf}_1(i)$, which is introduced to represent the probability of the occurrence of $\i{EnterFailed}$)} at time step $i$ if there is some abnormality at time step $i$. { Similarly we have
\begin{align}
\nonumber w_2\ &:\ \i{Pf}_2(i)\\
\nonumber \alpha\ &:\  Ab(\i{DropBook},i) \leftarrow \i{Pf}_2(i), Ab(i).\\
\nonumber w_3\ &:\ \i{Pf}_3(i)\\
\nonumber \alpha\ &:\  Ab(\i{PickupFailed},i) \leftarrow \i{Pf}_3(i), Ab(i).
\end{align}
}

When we describe the effect of actions, we need to specify ``no abnormality'' as part of the precondition of the effect: The location of the robot changes to room $r$ if it goes to room $r$ unless abnormality $\i{EnterFailed}$ occurs:
{\small
\[
 \alpha :\ \i{LocRobot}(r, i+1) \leftarrow Goto(r, i), not\ Ab(\i{EnterFailed}, i).
\]
}
The location of the book is the same as the location of the robot if the robot has the book:
\[
\alpha\ :\ \i{LocBook}(r, i) \leftarrow \i{LocRobot}(r, i), \i{HasBook}(T, i).
\]
The robot has the book if it is at the same location as the book and it picks up the book, unless abnormality $\i{PickupFailed}$ occurs:
{\small 
\begin{align}
\nonumber \alpha &:\ \i{HasBook}(\true, i+1) \leftarrow \i{PickupBook}(\true, i),\\
\nonumber & ~~~~~\i{LocRobot}(r, i), \i{LocBook}(r, i),not\ Ab(\i{PickupFailed}, i).
\end{align}
}
The robot loses the book if it puts down the book:
\[
\alpha\ :\ \i{HasBook}(\false, i+1) \leftarrow PutdownBook(\true, i).
\]
The robot loses the book if abnormality $DropBook$ occurs:
\[
\alpha\ :\ \i{HasBook}(\false, i+1) \leftarrow Ab(DropBook, i).
\]
The commonsense law of inertia for each fluent is specified by the following hard rules:
\begin{align}
\nonumber \alpha\ &:\ \{\i{LocRobot}(r, i+1)\} \leftarrow \i{LocRobot}(r, i), astep(i). \\
\nonumber \alpha\ &:\ \{\i{LocBook}(r, i+1)\} \leftarrow \i{LocBook}(r, i), astep(i). \\
\nonumber \alpha\ &:\ \{\i{HasBook}(b, i+1)\} \leftarrow \i{HasBook}(b, i), astep(i).  
\end{align}
For the lack of space, we skip the rules specifying the uniqueness and existence of fluents and actions, rules specifying that no two actions can occur at the same timestep, and rules specifying that the initial state and actions are exogenous. 

We add the hard rule
\[
\alpha\ :\ Ab(i) \ar astep(i)
\]
to enable abnormalities for each timestep $i$.

To use multiple action histories as the training data, we use the method from Section \ref{sec:multiple-evidence} and introduce an extra argument to every predicate, that represents the action history ID.

\NBB{multiple-evidence + incomplete stable model?}


We then provide a list of 12 transitions as the training data. For example, the first transition ({\tt ID} =1) tells us that the robot performed {\tt goto} action to room {\tt r2}, which failed.
\begin{lstlisting}
:- not loc_robot("r1",0,1). :- not loc_book("r2",0,1).
:- not hasBook("f",0,1).    :- not goto("r2",0,1).
:- not loc_robot("r1",1,1).
\end{lstlisting}

 
\NBB{Need to justify why [0,1] only} 

Among the training data, {\tt enter\_failed} occurred 1 time out of 4 attempts, {\tt pickup\_failed} occurred 2 times out of 4 attempts, and {\tt drop\_book} occurred 1 time out of 4 attempts. {The transitions are partially observed data in the sense that they specify only some of the fluents and actions; other facts about fluents, actions and abnormalities have to be inferred.}

Note that this program is $(|A|+1)$-coherent, where $|A|$ is the number of actions (i.e., $Goto$, $PickupBook$ and $DropBook$) and $1$ is for no actions. We execute gradient ascent learning with 50 learning iterations and 50 sampling iterations for each learning iteration. 
The weights learned are
\begin{lstlisting}
Rule 1:  -1.084    Rule 2:  -1.064    Rule 3:  -0.068
\end{lstlisting}
The probability of each abnormality can be computed from the weights as follows:
{\small
\[
P({\tt enter\_failed}) = \frac{exp(-1.084)}{exp(-1.084) + 1} \approx 0.253
\]
\[
P({\tt drop\_book}) = \frac{exp(-1.064)}{exp(-1.064) + 1} \approx 0.257
\]
\[
P({\tt pickup\_failed}) = \frac{exp(-0.068)}{exp(-0.068) + 1} \approx 0.483
\]
}


The learned weights of {\tt pf} atoms indicate the probability of the action failure when some abnormal situation {\tt ab(I, ID)} happens. This allows us to perform probabilistic diagnostic reasoning in which parameters are learned from the histories of actions. For example, suppose the robot and the book were initially at {\tt r1}. The robot executed the following actions to deliver the book from {\tt r1} to {\tt r2}: pick up the book; go to {\tt r2}; put down the book.
However, after the execution, it observes that the book is not at {\tt r2}. What was the problem?

Executing system {\sc lpmln2asp} on this encoding tells us that the most probable reason is that the robot fails at picking up the book. However, if we add that the robot itself is also not at {\tt r2}, then {\sc lpmln2asp} computes the most probable stable model to be the one that has the robot failed at entering~{\tt r2}.








\section{Conclusion}

The work presented relates answer set programming to {\em learning from data}, which has been under-explored, with some exceptions like \cite{law14inductive,nickles16atool}. Via $\lpmln$,  learning methods developed for Markov Logic can be adapted to find the weights of rules under the stable model semantics, utilizing answer set solvers for performing MCMC sampling. Rooted in the stable model semantics, $\lpmln$ learning is useful for learning parameters for programs modeling knowledge-rich domains. Unlike MC-SAT for Markov Logic, MC-ASP allows us to infer the missing part of the data guided by the stable model semantics. Overall, the work paves a way for a knowledge representation formalism to embrace machine learning methods.


The current $\lpmln$ learning implementation is a prototype with the  computational bottleneck in the uniform sampler, which is used as a blackbox. Unlike the work in machine learning, sampling has not been much considered in the context of answer set programming, and even the existing sampler we adopted was not designed for iterative calls as required by the MCMC sampling method. This is where we believe a significant performance increase can be gained. Using the idea such as constrained sampling \cite{meel16constrained} may enhance the solution quality and the scalability of the implementation, which is left for future work.


PrASP \cite{nickles14probabilistic} is related to $\lpmln$ in the sense that it is also a probabilistic extension of ASP. Weight learning in PrASP is very similar to weight learning in $\lpmln$: a variation of gradient ascent is used to update the weights so that the weights converge to a value that maximizes the probability of the training data. In PrASP setting, it is a problem that the gradient of the probability of the training data cannot be expressed in a closed form, and is thus hard to compute. The way how PrASP solves this problem is to approximate the gradient by taking the difference between the probability of training data with current weight and with current weight slightly incremented. The probability of training data, given fixed weights, is computed with inference algorithms, which typically involve sampling methods.

In this paper, we only considered $\lpmln$ weight learning with the basic gradient ascent method. There are several advanced weight learning techniques and sophisticated problem settings used for MLN weight learning that can possibly be adapted to $\lpmln$. For example, \cite{lowd07efficient} discussed some enhancement to the basic gradient ascent, \cite{khot11learning} proposed a method for learning the structure and the weights simultaneously,  and \cite{Mittal16FineGW} discussed how to automatically identify  clusters of ground instances of a rule and learn different weight for each of these clusters.


\bigskip\noindent
{\bf Acknowledgments:} \ \ 
We are grateful to Zhun Yang and the anonymous referees for their useful comments. This work was partially supported by the National Science Foundation under Grants IIS-1526301 and IIS-1815337.

\bibliographystyle{aaai}

\onecolumn

\lstset{
   basicstyle=\ttfamily,
   basewidth=0.5em,
   numbers=none,
   numberstyle=\tiny,  
   stringstyle=\small\ttfamily,
   showspaces=false,
   showstringspaces=false
}

\appendix
\section{Proofs}

\subsection{Proof of Theorem \ref{thm:mc-sat-lpmln-pnt-correctness}}

\begin{lemma}\label{lem:softonly-general}
For any $\lpmln$ program $\Pi$ and a probabilistic stable model $I$ of $\Pi$, we have
\[
P_{\Pi}(I) = \frac{exp(\underset{w:R\in \Pi^{\rm soft}_I}{\sum}w)}{Z}
\]
where
\[
Z = \underset{\text{$J$ is a stable model of $\Pi$ }}{\sum}exp(\underset{w:R\in \Pi^{\rm soft}_J}{\sum}w)
\]
\end{lemma}
\begin{proof}
Let $k$ be the maximum number of hard rules in $\Pi$ that any interpretation can satisfy. For any interpretation $J$, we use $J\vDash_{SM} \Pi$ as an abbreviation of ``$J$ is a probabilistic stable model of $\Pi$''.

By definition we have
\begin{align}
\nonumber P_{\Pi}(I) &= \underset{\alpha\to\infty}{\lim}\frac{exp(\underset{w:R\in \Pi_I}{\sum}w)}{\underset{J\vDash_{SM}\Pi}{\sum}exp(\underset{w:R\in \Pi_J}{\sum}w)}
\end{align}
Splitting the denominator into two parts: those $J$'s that satify $k$ hard rules in $\Pi$ and those that satisfy less hard rules, and extracting the weights of $k$ hard rules, $k\alpha$, we have
\begin{align}
\nonumber P_{\Pi}(I) &= \underset{\alpha\to\infty}{\lim}\frac{exp(\underset{w:R\in \Pi_I}{\sum}w)}{exp(k\alpha)\underset{\substack{J\vDash_{SM}\Pi\\ |\Pi^{\rm hard}_J|=k}}{\sum}exp(\underset{w:R\in \Pi^{\rm soft}_J}{\sum}w) + \underset{\substack{J\vDash_{SM}\Pi\\ |\Pi^{\rm hard}_J|<k}}{\sum}exp(|\Pi^{\rm hard}_J|\cdot \alpha)exp(\underset{w:R\in \Pi^{\rm soft}_J}{\sum}w)}
\end{align}
Let $k^\prime$ denote the number of hard rules that $I$ satisfy. We have
\begin{align}
\nonumber P_{\Pi}(I) &= \underset{\alpha\to\infty}{\lim}\frac{exp(k^\prime\alpha)exp(\underset{w:R\in \Pi^{\rm soft}_I}{\sum}w)}{exp(k\alpha)\underset{\substack{J\vDash_{SM}\Pi\\ |\Pi^{\rm hard}_J|=k}}{\sum}exp(\underset{w:R\in \Pi^{\rm soft}_J}{\sum}w) + \underset{\substack{J\vDash_{SM}\Pi\\ |\Pi^{\rm hard}_J|<k}}{\sum}exp(|\Pi^{\rm hard}_J|\cdot\alpha)exp(\underset{w:R\in \Pi^{\rm soft}_J}{\sum}w)}
\end{align}
Dividing both the numerator and the denominator by $exp(k\alpha)$, we get
\begin{align}
\nonumber P_{\Pi}(I) &= \underset{\alpha\to\infty}{\lim}\frac{\frac{exp(k^\prime\alpha)}{exp(k\alpha)}exp(\underset{w:R\in \Pi^{\rm soft}_I}{\sum}w)}{\underset{\substack{J\vDash_{SM}\Pi\\ |\Pi^{\rm hard}_J|=k}}{\sum}exp(\underset{w:R\in \Pi^{\rm soft}_J}{\sum}w) + \frac{1}{exp(k\alpha)}\underset{\substack{J\vDash_{SM}\Pi\\ |\Pi^{\rm hard}_J|<k}}{\sum}exp(|\Pi^{\rm hard}_J|\cdot\alpha)exp(\underset{w:R\in \Pi^{\rm soft}_J}{\sum}w)}
\end{align}
We argue that $k^\prime=k$: Since $k$ is the maximum number of hard rules an interpretation can satisfy, $k^\prime \leq k$; Suppose $k^\prime < k$. Then the above expression evaluates to $0$, contradicting the fact that $I$ is a probabilistic stable model of $\Pi$. Following the same argument, it can be seen that any stable model of $\Pi$ satisfy $k$ hard rules. So $k^\prime=k$, and thus we have
\begin{align}
\nonumber P_{\Pi}(I) &= \underset{\alpha\to\infty}{\lim}\frac{exp(\underset{w:R\in \Pi^{\rm soft}_I}{\sum}w)}{\underset{\substack{J\vDash_{SM}\Pi\\ |\Pi^{\rm hard}_J|=k}}{\sum}exp(\underset{w:R\in \Pi^{\rm soft}_J}{\sum}w) + \frac{1}{exp(k\alpha)}\underset{\substack{J\vDash_{SM}\Pi\\ |\Pi^{\rm hard}_J|<k}}{\sum}exp(|\Pi^{\rm hard}_J|\cdot\alpha)exp(\underset{w:R\in \Pi^{\rm soft}_J}{\sum}w)}\\
\nonumber  &= \underset{\alpha\to\infty}{\lim}\frac{exp(\underset{w:R\in \Pi^{\rm soft}_I}{\sum}w)}{\underset{\substack{J\vDash_{SM}\Pi\\ |\Pi^{\rm hard}_J|=k}}{\sum}exp(\underset{w:R\in \Pi^{\rm soft}_J}{\sum}w) + \underset{\substack{J\vDash_{SM}\Pi\\ |\Pi^{\rm hard}_J|<k}}{\sum}\frac{exp(|\Pi^{\rm hard}_J|\cdot\alpha)}{exp(k\alpha)}exp(\underset{w:R\in \Pi^{\rm soft}_J}{\sum}w)}
\end{align}
For those $J$ that satisfy less than $k$ hard rules, $\frac{exp(|\Pi^{\rm hard}|)}{exp(k\alpha)} \leq k-1$, so we have
\begin{align}
\nonumber P_{\Pi}(I) &=\frac{exp(\underset{w:R\in \Pi^{\rm soft}_I}{\sum}w)}{\underset{\substack{J\vDash_{SM}\Pi\\
|\Pi^{\rm hard}_J|=k}}{\sum}exp(\underset{w:R\in \Pi^{\rm soft}_J}{\sum}w)}
\end{align}
Since all stable model of $\Pi$ satisfy $k$ hard rules, we have
\begin{align}
\nonumber P_{\Pi}(I) &=\frac{exp(\underset{w:R\in \Pi^{\rm soft}_I}{\sum}w)}{\underset{\substack{J\vDash_{SM}\Pi}}{\sum}exp(\underset{w:R\in \Pi^{\rm soft}_J}{\sum}w)}.
\end{align}
\qed
\end{proof}

\noindent{\bf Theorem~\ref{thm:mc-sat-lpmln-pnt-correctness} \optional{thm:mc-sat-lpmln-pnt-correctness}}\
\ 
{\sl
The Markov chain generated by MC-ASP satisfies ergodicity and detailed balance.
}

\begin{proof}
{\bf Ergodicity}  Firstly, for any subset $M$ of rules generated at step 2 in Algortihm \ref{alg:mc-asp}, the previous sample $I^{j-1}$ is always a stable model that satisifies no rules in $M$, which means at least one sample can be produced at any sampling step. Secondly, it is always possible that $M$ is an empty set. All stable models of $\Pi$ are possible to be selected when $M$ is empty set. Thus every stable model is reachable from every stable model.

{\bf Detailed Balance} For any (probabilistic) stable models $X$ and $Y$ of $\Pi$, let $Q(X\rightarrow Y)$ denote the transition probability from $X$ to $Y$ (i.e., the probability that the next sample is $Y$ given that the current sample is $X$), and let let $Q(X\rightarrow^{M} Y)$ denote the transition probability from $X$ to $Y$ through a particular subset of rules as the set $M$ at step 2 in Algorithm \ref{alg:mc-asp}.  Let $Q_M(X)$ be the probability of choosing $X$ from $M$. To show $P_{\Pi}(X)Q(X\rightarrow Y) = P_{\Pi}(Y)Q(Y\rightarrow X)$, we prove a stronger equation $P_{\Pi}(X)Q(X\rightarrow^{M} Y) = P_{\Pi}(Y)Q(Y\rightarrow^{M} X)$ for any $M\subseteq (\o{\Pi}^{soft}\setminus\o{\Pi}^{soft}_X) \cap (\o{\Pi}^{soft}\setminus\o{\Pi}^{soft}_Y)$.
By Lemma \ref{lem:softonly-general}, we have
\[
P(X) = \frac{1}{Z}\underset{R_i\in \o{\Pi^{\rm soft}\setminus\Pi^{\rm soft}_{X}}}{\prod}e^{-w_i}
\]
and
\[
Q(X \rightarrow^{M} Y)= \prod_{R_i\in(\o{\Pi^{\rm soft}\setminus\Pi^{\rm soft}_{X}})\setminus M}e^{w_i}\cdot \prod_{R_i\in M}(1-e^{w_i})\cdot Q_M(Y).
\]
Consequently we have
\begin{align}
\nonumber & P(X)Q(X\rightarrow^{M} Y) \\
\nonumber = & \frac{1}{Z}\underset{R_i\in \o{\Pi^{\rm soft}\setminus\Pi^{\rm soft}_{X}}}{\prod}e^{-w_i} \cdot \prod_{R_i\in(\o{\Pi^{\rm soft}\setminus\Pi^{\rm soft}_{X}})\setminus M}e^{w_i}\cdot \prod_{R_i\in M}(1-e^{w_i})\cdot Q_M(Y)\\
\nonumber = & \frac{1}{Z} \cdot \prod_{R_i\in M}e^{-w_i}\cdot \prod_{R_i\in M}(1-e^{w_i})\cdot Q_M(Y).
\end{align}
It can be seen that $Q_M(X)=Q_M(Y)$ as any stable model of $\Pi$ that satisfies $M$ is drawn with the same probability. So we have
\begin{align}
\nonumber & P(X)Q(X\rightarrow^{M} Y) \\
\nonumber = & \frac{1}{Z} \cdot \prod_{R_i\in M}e^{-w_i}\cdot \prod_{R_i\in M}(1-e^{w_i})\cdot Q_M(X)\\
\nonumber = & \frac{1}{Z}\underset{R_i\in \o{\Pi^{\rm soft}_{Y}}}{\prod}e^{-w_i} \cdot \prod_{R_i\in\o{\Pi^{\rm soft}_{Y}}\setminus M}e^{w_i}\cdot \prod_{R_i\in M}(1-e^{w_i})\cdot Q_M(Y)\\
\nonumber = &P(Y)Q(Y\rightarrow^{M} X).
\end{align}
\qed
\end{proof}

\subsection{Proof of Theorem \ref{thm:pi-neg-equivalence}}

\begin{lemma}\label{lem:neg-1-1}
Assume $\sm^\prime[\Pi]$ is not empty. For any interpretation $I$ of $\Pi$,
\[
{\rm neg}(I) = I\cup\{neg(i, {\bf x}) \mid I\nvDash H({\bf x})\leftarrow B({\bf x}), w_i: H({\bf x})\leftarrow B({\bf x})\in \Pi\}
\]
is a $1-1$ correspondence between $\sm[\Pi]$ and $\sm[\Pi^{neg}]$.
\end{lemma}
\begin{proof}
We divide the ground program obtained from $\Pi^{\rm neg}$ into three parts:
\[
ORIGIN(\Pi)\cup NEGDEF(\Pi) \cup NEG(\Pi)
\]
where
\begin{align}
\nonumber ORIGIN(\Pi) =& \{w: H({\bf x})\leftarrow B({\bf x}) \mid w: H({\bf x})\leftarrow B({\bf x}) \in \Pi, w\leq 0\}\cup\\
\nonumber & \{0: H({\bf x})\leftarrow B({\bf x}) \mid w: H({\bf x})\leftarrow B({\bf x}) \in \Pi, w> 0\},
\end{align}
\[
NEGDEF(\Pi) = \{\alpha:{\tt neg}(i, {\bf x})\leftarrow B({\bf x}), {\tt not}\ H({\bf x})\mid w:H({\bf x})\leftarrow B({\bf x})\in \Pi, w>0\}
\]
and
\[
NEG(\Pi) = \{-w:\ \leftarrow {\tt not}\ {\tt neg}(i, {\bf x})\mid w_i:R_i\in \Pi, w>0\}.
\]
Let $\sigma$ be the signature of $\Pi$, and $\sigma_{neg}$ be the set
\[
\{{\tt neg}(i, {\bf c})\mid w: H({\bf c})\leftarrow B({\bf c})\in Gr(\Pi), w > 0\}.
\]

For any interpretation $I$ of $\Pi$, consider $\o{\Pi^{\rm neg}}_{{\rm neg}(I)}$. From the construction of ${\rm neg}(I)$, we have
\[
\o{\Pi^{\rm neg}}_{{\rm neg}(I)} = \o{ORIGIN(\Pi)}_I\cup \o{NEGDEF(\Pi)}_{{\rm neg}(I)} \cup \o{NEG(\Pi)}_{{\rm neg}(I)}
\]
It can be seen that 
\begin{itemize}
\item each strongly connected component of the dependency graph of $\o{ORIGIN(\Pi)}_I\cup \o{NEGDEF(\Pi)}_{{\rm neg}(I)} \cup \o{NEG(\Pi)}_{{\rm neg}(I)}$ w.r.t. $\sigma\cup \sigma_{neg}$ is a subset of $\sigma$ or a subset of $\sigma_{neg}$;
\item no atom in $\sigma_{neg}$ has a strictly positive occurrence in $\o{ORIGIN(\Pi)}_I$;
\item no atom in $\sigma$ has a strictly positive occurrence in $\o{NEGDEF(\Pi)}_{{\rm neg}(I)} \cup \o{NEG(\Pi)}_{{\rm neg}(I)}$
\end{itemize}
Thus, according to the splitting theorem, ${\rm neg}(I)$ is a stable model of $\o{\Pi^{\rm neg}}({\rm neg}(I))$ if and only if ${\rm neg}(I)$ is a stable model of $\o{ORIGIN(\Pi)}_I$ w.r.t. $\sigma$ and is a stable model of $\o{NEGDEF(\Pi)}_{{\rm neg}(I)} \cup \o{NEG(\Pi)}_{{\rm neg}(I)}$ w.r.t. $\sigma_{neg}$.

Suppose $I$ is a probabilistic stable model of $\Pi$. We will show that ${\rm neg}(I)$ is a stable model of $\o{\Pi^{\rm neg}}({\rm neg}(I))$.
\begin{itemize}
\item {\bf ${\rm neg}(I)$ is a stable model of $\o{ORIGIN(\Pi)}_I$ w.r.t. $\sigma$.} By definition, $I$ is a stable model of $\o{\Pi}_I$. Since $\o{ORIGIN(\Pi)}_I=\o{\Pi}_I$ and $I$ and ${\rm neg(I)}$ agrees on $\sigma$, ${\rm neg}(I)$ is a stable model of $\o{ORIGIN(\Pi)}_I$ w.r.t. $\sigma$.
\item {\bf ${\rm neg}(I)$ is a stable model of $\o{NEGDEF(\Pi)}_{{\rm neg}(I)} \cup \o{NEG(\Pi)}_{{\rm neg}(I)}$ w.r.t. $\sigma_{neg}$.} Clearly, ${\rm neg}(I)$ satisfies $\o{NEGDEF(\Pi)}_{{\rm neg}(I)} \cup \o{NEG(\Pi)}_{{\rm neg}(I)}$. From the construction of ${\rm neg}(I)$, ${\rm neg}(I)$ satisfies ${\tt neg}(i, {\bf x})$ only if  ${\rm neg}(I)$ does not satisfy $H({\bf x})\leftarrow B({\bf x})$. This means ${\rm neg}(I)$ satisfies
\[
{\tt neg}(i, {\bf c})\rightarrow B({\bf c}), {\tt not}\ H({\bf c})
\]
for all rules $H({\bf c})\leftarrow B({\bf c})$ in $\Pi$. This is the completion of $\o{NEGDEF(\Pi)}_{{\rm neg}(I)} \cup \o{NEG(\Pi)}_{{\rm neg}(I)}$ w.r.t. $\sigma_{neg}$. Obviously $\o{NEGDEF(\Pi)}_{{\rm neg}(I)} \cup \o{NEG(\Pi)}_{{\rm neg}(I)}$ is tight. So ${\rm neg}(I)$ is a stable model of $\o{NEGDEF(\Pi)}_{{\rm neg}(I)} \cup \o{NEG(\Pi)}_{{\rm neg}(I)}$ w.r.t. $\sigma_{neg}$.
\end{itemize}
Suppose $J$ is a probabilistic stable model of $\Pi^{neg}$. By definition, $J$ is a stable model of $\o{\Pi^{\rm neg}}({\rm neg}(I))$. By the splitting theorem, $J$ is a stable model of $\o{ORIGIN(\Pi)}_I$ w.r.t. $\sigma$. Let $I$ be the interpretation of $\Pi$ obtained by dropping atoms in $\sigma_{neg}$ from $J$. Since $\o{ORIGIN(\Pi)}_I = \o{\Pi}_I$ and $I$ agrees with $J$ on $\sigma$, $I$ is a stable model of $\o{\Pi}_I$, and thus is a stable model of $\Pi$.
\qed
\end{proof}

\noindent{\bf Theorem~\ref{thm:pi-neg-equivalence} \optional{thm:pi-neg-equivalence}}\
\ 
{\sl
When $\sm[\Pi]$ is not empty, the program $\Pi^{neg}$ specifies the same probability distribution as the program $\Pi$.
}

\begin{proof}
We show that $P_{\Pi^{neg}}(neg(I)) = P_{\Pi}(I)$ for all interpretations $I$.

By Lemma \ref{lem:neg-1-1}, since ${\rm neg}(I)$ defines a $1-1$ correspondence between the probabilistic stable models of $\Pi$ and $\Pi^{neg}$, when $I$ is not a probabilistic stable model of $\Pi$, ${\rm neg}(I)$ is not a probabilistic stable model of $\Pi^{\rm neg}$, and vice versa. So $P_{\Pi}(I) = P_{\Pi^{\rm neg}}({\rm neg}(I)) = 0$.

For any program $\Pi$, we use $n_{\Pi, i}(I)$ to denote the number of ground instances of rule $i$ that is satisfies by $I$, $m_{\Pi, i}(I)$ to denote the number of ground instances of rule $i$ that is not satisfied by $I$, and $N_{\Pi, i}$ to denote the total number of ground instances of rule $i$.

When $I$ is a probabilistic stable model of $\Pi$, we have
\begin{align}
\nonumber &W^\prime_{\Pi}(I) \\
\nonumber =\ & exp(\sum_{w_i:R_i\in \Pi^{\rm soft}}w_in_{\Pi, i}(I))\\
\nonumber =\ &\text{(Splitting rules into the ones whose weights are positive and the ones whose weights are non-positive)}\\
\nonumber &exp(\sum_{w_i:R_i\in \Pi^{\rm soft}, w_i>0}w_in_{\Pi, i}(I)) \cdot exp(\sum_{w_i:R_i\in \Pi^{\rm soft}, w_i\leq 0}w_in_{\Pi, i}(I)).
\end{align}

\begin{align}
\nonumber &W^\prime_{\Pi^{neg}}({\rm neg}(I)) \\
\nonumber =\ & exp(\sum_{w_i:R_i\in (\Pi^{neg})^{\rm soft}}w_in_{\Pi^{neg}, i}({\rm neg}(I)))\\
\nonumber =\ &\text{(Splitting rules into the ones whose weights are positive and the ones whose weights are non-positive)}\\
\nonumber & exp(\sum_{w_i:R_i\in \Pi^{\rm soft}, w_i\leq 0}w_in_{\Pi, i}(I))\cdot exp(\sum_{w_i:R_i\in \Pi^{\rm soft}, w_i > 0}-w_i m_{\Pi, i}(I))\\
\nonumber =\ & \frac{exp(\sum_{w_i:R_i\in \Pi^{\rm soft}, w_i> 0}w_iN_{\Pi, i})\cdot exp(\sum_{w_i:R_i\in \Pi^{\rm soft}, w_i\leq 0}w_in_{\Pi, i})\cdot exp(\sum_{w_i:R_i\in \Pi^{\rm soft}, w_i> 0}-w_in_{\Pi, i}(I))}{exp(\sum_{w_i:R_i\in \Pi^{\rm soft}, w_i> 0}w_iN_{\Pi, i})}\\
\nonumber =\ & \frac{exp(\sum_{w_i:R_i\in \Pi^{\rm soft}, w_i> 0}w_iN_{\Pi, i})\cdot exp(\sum_{w_i:R_i\in \Pi^{\rm soft}, w_i> 0}-w_in_{\Pi, i}(I)) \cdot exp(\sum_{w_i:R_i\in \Pi^{\rm soft}, w_i\leq 0}w_in_{\Pi, i})}{exp(\sum_{w_i:R_i\in \Pi^{\rm soft}, w_i> 0}w_iN_{\Pi, i})}\\
\nonumber =\ & \frac{1}{exp(\sum_{w_i:R_i\in \Pi^{\rm soft}, w_i> 0}w_iN_{\Pi, i})} exp(\sum_{w_i:R_i\in \Pi^{\rm soft}, w_i> 0}w_in_{\Pi, i}(I)) \cdot exp(\sum_{w_i:R_i\in \Pi^{\rm soft}, w_i\leq 0}w_in_{\Pi, i}(I))\\
\nonumber \propto\ &W^\prime_{\Pi}(I).
\end{align}

Consequently, we have
\[
P^\prime_{\Pi}(I) = P^\prime_{\Pi^{neg}}({\rm neg}(I)).
\]
Since $\sm^\prime[\Pi]$ is not empty, by Proposition 2 in \cite{lee16weighted}, we have
\[
P_{\Pi}(I) = P_{\Pi^{neg}}({\rm neg}(I)).
\]
\qed
\end{proof}

\subsection{Proof of Theorem \ref{thm:multiple-evidence}}

\noindent{\bf Theorem~\ref{thm:multiple-evidence} \optional{thm:multiple-evidence}}\
\ 
{\sl
For any parameterized $\lpmln$ program $\hat{\Pi}$, its stable models $I_1,\dots, I_m$ and $I$ as defined as in~\eqref{multiple-sm}, we have
\[
\underset{{\bf w}}{\rm argmax}\ P_{{\hat{\Pi}^m}({\bf w})}(I)  =
\underset{{\bf w}}{\rm argmax}\ 
    \underset{i\in\{1, \dots, m\}}{\prod}P_{\hat{\Pi}({\bf w})}(I_i).
\]
}

\begin{proof}
For any weight vector ${\bf w}$, we show
\[
P_{\hat{\Pi}^m({\bf w})}(D) = \underset{i\in\{1, \dots, m\}}{\prod}P_{\hat{\Pi}({\bf w})}(D_i)
\]
by induction. For any integer $1\leq u\leq m$, we use $D^u$ to denote the interpretation 
\[
D^u = \{p({\bf x}, j) \mid p({\bf x}, j)\in D, j \leq u\}
\]
Note that $D^m = D$.

{\bf Base Case:} Suppose $m = 1$. It is trivial that we have
\[
P_{\hat{\Pi}^1({\bf w})}(D^1) = \underset{i\in\{1\}}{\prod}P_{\hat{\Pi}({\bf w})}(D_i)
\]

For $m > 1$, as I.H., we assume
\[
P_{\hat{\Pi}^{m-1}({\bf w})}(D^{m-1}) = \underset{i\in\{1, \dots, m-1\}}{\prod}P_{\hat{\Pi}({\bf w})}(D_i)
\]
We divide $\hat{\Pi}^{m}({\bf w})$ into two disjoint subsets:
\[
\hat{\Pi}^{m}({\bf w}) = \hat{\Pi}^{m-1}({\bf w}) \cup \hat{\Pi}({\bf w})[x=m]
\]
where $\hat{\Pi}[x=m]({\bf w})$ is the program obtained from $\hat{\Pi}({\bf w})$ by appending one more argument whose value is $m$ to the list of argument of every occurrence of every predicate in $\hat{\Pi}({\bf w})$.
Clearly, the intersection between the set of atoms that occur in $gr(\hat{\Pi}^{m-1}({\bf w}))$ and that occur $gr(\hat{\Pi}({\bf w})[x=m])$ is empty. According to Definition 12 in \cite{wang18splitting}, $gr(\hat{\Pi}^{m}({\bf w}))$ is independently divisible and $gr(\hat{\Pi}^{m-1}({\bf w}))$ and $gr(\hat{\Pi}({\bf w})[x=m])$ are independent programs w.r.t. $gr(\hat{\Pi}^{m}({\bf w}))$.

By Corollary 3 in \cite{wang18splitting}, we have
\begin{align}
\nonumber P_{\hat{\Pi}^{m}({\bf w})}(D^{m}) &= P_{\hat{\Pi}^{m-1}({\bf w})}(D^{m-1}) \cdot P_{\hat{\Pi}({\bf w})[x=m]}(D^{m}\setminus D^{m-1}) \\
\nonumber &= P_{\hat{\Pi}^{m-1}({\bf w})}(D^{m-1}) \cdot P_{\hat{\Pi}({\bf w})[x=m]}(D_{m})
\end{align}
By I.H., we have
\begin{align}
\nonumber P_{\hat{\Pi}^{m}({\bf w})}(D^{m}) &= \underset{i\in\{1, \dots, m-1\}}{\prod}P_{\hat{\Pi}({\bf w})}(D_i) \cdot P_{\hat{\Pi}({\bf w})[x=m]}(D_{m}) \\
\nonumber &= \underset{i\in\{1, \dots, m\}}{\prod}P_{\hat{\Pi}({\bf w})}(D_i).
\end{align}
\qed
\end{proof}

\subsection{Proof of Theorem \ref{thm:lpmln-learn-to-mln-learn}}

\noindent{\bf Theorem~\ref{thm:lpmln-learn-to-mln-learn} \optional{thm:lpmln-learn-to-mln-learn}}\
\ 
{\sl
Let ${\rm L}$ be the Markov Logic Network $Comp(\Pi)$ and let $E$ be a ground formula (as the training data).
When $\sm[\Pi]$ is not empty, 
\[
\underset{{\bf {\bf w}}}{\rm argmax}\ P_{\hat{\Pi}({\bf w})}(E) = 
\underset{{\bf w}}{\rm argmax}\ P_{\hat{\rm L}({\bf w})}(E).
\]
($\hat{\rm L}$ is a parameterized Markov Logic Network obtained from ${\rm L}$.) 
}

\begin{proof}
Easily follows from Theorem 3 in \cite{lee16weighted}.\qed
\end{proof}

\subsection{Proof of Theorem \ref{cor:lpmln-learn2problog-learn} and Theorem \ref{cor:coherent-learning} }
\begin{lemma}\label{lem:lpmln2problog}
For any 1-coherent $\lpmln$ program $\langle PF, P, {\bf w}\rangle$, we have
\[
P_{\langle PF, P, {\bf w}\rangle}(I) = P_{\langle PF, P, {\bf pr}\rangle}(I)
\]
for any interpretation $I$ and ${\bf w} = {\bf w}^{\bf pr}$
\end{lemma}
\begin{proof}
Similar to the proof of Theorem 5 in \cite{lee16weighted}.\qed
\end{proof}

\noindent{\bf Theorem~\ref{cor:lpmln-learn2problog-learn} \optional{cor:lpmln-learn2problog-learn}}\
\ 
{\sl
For any 1-coherent parameterized $\lpmln$ program $\langle PF, P, {\bf w}\rangle$ and any interpretation $T$ (as the training data), we have
\begin{align}
\nonumber &{\bf w} = \underset{{\bf w}}{\rm argmax}\ P_{\langle PF, P, {\bf w}\rangle}(T)\\
\nonumber &\text{~~~~~~~~~~if and only if}\\
\nonumber &{\bf w} = {\bf w}^{\bf pr}\text{ and } {\bf pr} = \underset{{\bf pr}}{\rm argmax}\ P_{\langle PF, P, {\bf pr}\rangle}(T).
\end{align}
}

\begin{proof}
Easily follows from Lemma \ref{lem:lpmln2problog}.
\end{proof}

\begin{prop}\label{thm:mvpp-skip-normalization}
For any k-coherent $\lpmln$ program $\Pi=\langle PF, \Pi^{hard}, {\bf w}\rangle$ and any interpretation $I$, we have
\[
P_{\Pi}(I)= \frac{1}{k\cdot \underset{pf_j\in PF}{\prod}(1+e^{w_j})}W_{\Pi}(I).
\]
\end{prop}

\begin{proof}
We show that the normalization factor is constant $k\cdot \underset{pf_j\in PF}{\prod}(1+e^{w_j})$, i.e., 
\[
\underset{\text{$I$ is an interpretation of $\Pi$}}{\sum} W_{\Pi}(I) = k\cdot \underset{pf_j\in PF}{\prod}(1+e^{w_j}).
\]
Let $pf_1, \dots, pf_m\in PF$ be the soft atoms. Let $TC_{\Pi}$ be the set of all truth assignments to atoms in $PF$. 

\begin{align}
\nonumber &\underset{\text{$I$ is an interpretation of $\Pi$}}{\sum} W_{\Pi}(I) \\
\nonumber =&\underset{I\in SM[\Pi]}{\sum} W_{\Pi}(I)\\
\nonumber =&\underset{tc\in TC_{\Pi}}{\sum} k\cdot \underset{tc\vDash pf_i}{\prod}exp(w_i)\cdot \underset{tc\nvDash pf_j}{\prod}exp(0)\\
\nonumber =&k\underset{tc\in TC_{\Pi}}{\sum}\cdot \underset{tc\vDash pf_i}{\prod}exp(w_i)\cdot \underset{tc\nvDash pf_j}{\prod}exp(0)\\
\nonumber =&k\cdot(e^{w_1}\underset{\substack{tc\in TC_{\Pi}\\ tc\vDash pf_i\\i \neq 1}}{\prod}e^{w_i}\cdot\underset{\substack{tc\in TC_{\Pi}\\ tc\nvDash pf_i\\i \neq 1}}{\prod}e^0 + e^0\underset{\substack{tc\in TC_{\Pi}\\ tc\vDash pf_i\\i \neq 1}}{\prod}e^{w_i}\cdot\underset{\substack{tc\in TC_{\Pi}\\ tc\nvDash pf_i\\i \neq 1}}{\prod}e^0)\\
\nonumber =&k\cdot(e^{w_2}\cdot(e^{w_1}\underset{\substack{tc\in TC_{\Pi}\\ tc\vDash pf_i\\i \neq 1\\i\neq 2}}{\prod}e^{w_i}\cdot\underset{\substack{tc\in TC_{\Pi}\\ tc\nvDash pf_i\\i \neq 1\\i\neq 2}}{\prod}e^0 + e^0\underset{\substack{tc\in TC_{\Pi}\\ tc\vDash pf_i\\i \neq 1\\i \neq 2}}{\prod}e^{w_i}\cdot\underset{\substack{tc\in TC_{\Pi}\\ tc\nvDash pf_i\\i \neq 1\\i\neq 2}}{\prod}e^0) + e^{0}\cdot(e^{w_1}\underset{\substack{tc\in TC_{\Pi}\\ tc\vDash pf_i\\i \neq 1\\i\neq 2}}{\prod}e^{w_i}\cdot\underset{\substack{tc\in TC_{\Pi}\\ tc\nvDash pf_i\\i \neq 1\\i\neq 2}}{\prod}e^0 + e^0\underset{\substack{tc\in TC_{\Pi}\\ tc\vDash pf_i\\i \neq 1\\i \neq 2}}{\prod}e^{w_i}\cdot\underset{\substack{tc\in TC_{\Pi}\\ tc\nvDash pf_i\\i \neq 1\\i\neq 2}}{\prod}e^0))\\
\nonumber =&\dots\\
\nonumber =&k\cdot(\underset{p_{1}\in\{e^{w_{1}}, 1\}}{\sum}p_{1}\dots\underset{p_2\in\{e^{w_2}, 1\}}{\sum}p_{2}\underset{p_{m-1}\in\{e^{w_{m-1}}, 1\}}{\sum}p_{m-1}\cdot(e^{w_m}+1))\\
\nonumber =&k\cdot(e^{w_m}+1)(\underset{p_{1}\in\{e^{w_{1}}, 1\}}{\sum}p_{1}\dots\underset{p_2\in\{e^{w_2}, 1\}}{\sum}p_{2}\underset{p_{m-1}\in\{e^{w_{m-1}}, 1\}}{\sum}p_{m-1})\\
\nonumber =&k\cdot(e^{w_m}+1)(e^{w_{m-1}}+1)(\underset{p_{1}\in\{e^{w_{1}}, 1\}}{\sum}p_{1}\dots\underset{p_2\in\{e^{w_2}, 1\}}{\sum}p_{2}\underset{p_{m-2}\in\{e^{w_{m-2}}, 1\}}{\sum}p_{m-1})\\
\nonumber =&\dots\\
\nonumber =&k\cdot \underset{pf_j\in PF}{\prod}(1+e^{w_j}).
\end{align}
\qed
\end{proof}

\begin{prop}\label{prop:sm-prob-from-tc}
For any $k$-coherent $\lpmln$ program ${\bf \Pi}=\langle PF, \Pi^{hard}, {\bf w}\rangle$ and any interpretation $I$, we have
\[
Pr_{{\bf \Pi}}(I) = \begin{cases}
\frac{1}{k}\underset{c_i\in PF}{\prod}  Pr_{{\bf\Pi}}(c)^{m_i(I)} \cdot (1 - Pr_{{\bf\Pi}}(c_i))^{n_i(I)}  \\
 \hspace{2cm}   \text{if $I$ is a stable model of $\bf \Pi$}\\
0 \hspace{1.8cm} \text{otherwise}
\end{cases}
\]
\end{prop}

\begin{proof}
Easily proven from Proposition \ref{thm:mvpp-skip-normalization}.\qed
\end{proof}

\begin{prop}\label{prop:tc-independence}
For any $k$-coherent $\lpmln$ program ${\bf \Pi}=\langle PF, \Pi^{hard}, {\bf w}\rangle$, we have
\[
Pr_{{\bf\Pi}}(pf_i)=\frac{exp(w_i)}{exp(w_i) + 1}
\]
for any $pf_i\in PF$ and the corresponding weight $w_i$.
\end{prop}

\begin{proof}
By Proposition \ref{thm:mvpp-skip-normalization} we have
\begin{align}
\nonumber &Pr_{\bf \Pi}(pf_i) \\
\nonumber = & \underset{\substack{\text{$I$ is a stable model of ${\bf \Pi}$}\\I\vDash pf_i}}{\sum}\frac{\underset{I\vDash pf_j, pf_j\in PF}{\prod}e^{w_j}\cdot \underset{I\nvDash pf_j, pf_j\in PF}{\prod}e^0}{k\cdot \underset{pf_j\in PF}{\prod}(1+e^{w_j})}\\
\nonumber = &\frac{e^{w_i}}{e^{w_i} + 1} \cdot \underset{\substack{\text{$I$ is a stable model of ${\bf \Pi}$\ \ }\\I\vDash pf_i}}{\sum}\frac{\underset{I\vDash pf_j, pf_j\in PF,j\neq i}{\prod}e^{w_j}\cdot \underset{I\nvDash pf_j, pf_j\in PF}{\prod}e^0}{k\cdot \underset{\substack{pf_j\in PF\\j\neq i}}{\prod}(1+e^{w_j})}\\
\nonumber = &\frac{e^{w_i}}{e^{w_i} + 1} \cdot k\underset{\substack{\text{$I$ is a truth assignment to $PF\setminus\{pf_i\}$\ \ }\\}}{\sum}\frac{\underset{I\vDash pf_j, pf_j\in PF,j\neq i}{\prod}e^{w_j}\cdot \underset{I\nvDash pf_j, pf_j\in PF}{\prod}e^0}{k\cdot \underset{\substack{pf_j\in PF\\j\neq i}}{\prod}(1+e^{w_j})}\\
\nonumber = &\frac{e^{w_i}}{e^{w_i} + 1}\cdot \frac{\underset{\substack{\text{$I$ is a truth assignment to $PF\setminus\{pf_i\}$\ \ }\\}}{\sum}(\underset{I\vDash pf_j, pf_j\in PF,j\neq i}{\prod}e^{w_j}\cdot \underset{I\nvDash pf_j, pf_j\in PF}{\prod}e^0)}{ \underset{\substack{pf_j\in PF\\j\neq i}}{\prod}(1+e^{w_j})}\\
\nonumber = &\frac{e^{w_i}}{e^{w_i} + 1} \cdot 1\\
\nonumber = &\frac{e^{w_i}}{e^{w_i} + 1}.
\end{align}
\qed
\end{proof}

\noindent{\bf Theorem~\ref{cor:coherent-learning} \optional{cor:coherent-learning}}\
\ 
{\sl
For any $k$-coherent parameterized $\lpmln$ program $\langle PF, \Pi^{\rm hard}, {\bf w}\rangle$, and an interpretation $T$ as  the training data,  we have
\[
\underset{\bf w}{\rm argmax}\ P_{\langle PF, \Pi^{\rm hard}, {\bf w}\rangle}(T; {\bf w}) =(ln\frac{m_1(T)}{n_1(T)}, \dots, ln\frac{m_{|PF|}(T)}{n_{|PF|}(T)}).
\]
}
\begin{proof}
We have
\begin{align}
\nonumber &P_{\langle PF, \Pi^{\rm hard}, {\bf w}\rangle}(T; {\bf w})\\
\nonumber = & (\text{Proposition \ref{prop:sm-prob-from-tc}})\\
\nonumber  &\frac{1}{k}\underset{c_i\in PF}{\prod} Pr_{{\bf\Pi}}(c)^{m_i(I)} \cdot (1 - Pr_{{\bf\Pi}}(c_i))^{n_i(I)} \\
\nonumber = & (\text{Proposition \ref{prop:tc-independence}})\\
\nonumber  &  \frac{1}{k}\underset{c_i\in PF}{\prod} (\frac{exp(w_i)}{exp(w_i) + 1})^{m_i(I)} \cdot (1 - \frac{exp(w_i)}{exp(w_i) + 1})^{n_i(I)}
\end{align}
\begin{align}
\nonumber &lnP_{\langle PF, \Pi^{\rm hard}, {\bf w}\rangle}(T; {\bf w})\\
\nonumber = & ln\frac{1}{k} + \underset{c_i\in PF}{\sum} m_i(I)(w_i - ln(exp(w_i) + 1)) + \\
\nonumber & {n_i(I)}(ln 1 - ln (exp(w_i+1))
\end{align}

Since $P_{\langle PF, \Pi^{\rm hard}, {\bf w}\rangle}(T; {\bf w})$ is concave w.r.t. $w_i\in {\bf w}$, the value of $w_i$ that maximizes $P_{\langle PF, \Pi^{\rm hard}, {\bf w}\rangle}(T; {\bf w})$ can be obtained by solving
\[
\frac{\partial lnP_{\langle PF, \Pi^{\rm hard}, {\bf w}\rangle}(T; {\bf w})}{\partial w_i} = 0.
\]
For any $w_j\in {\bf w}$, we have
\begin{align}
\nonumber &\frac{\partial lnP_{\langle PF, \Pi^{\rm hard}, {\bf w}\rangle}(T; {\bf w})}{\partial w_j}\\
\nonumber =& m_j(I)(1-\frac{exp(w_j)}{exp(w_j) + 1}) - n_j(I)\frac{exp(w_j)}{exp(w_j)+1}
\end{align}
\vspace{0.2cm}
\[
\frac{\partial lnP_{\langle PF, \Pi^{\rm hard}, {\bf w}\rangle}(T; {\bf w})}{\partial w_i} = 0
\]
is equivalent to
\begin{align}
\nonumber & m_j(I)(1-\frac{exp(w_j)}{exp(w_j) + 1}) = n_j(I)\frac{exp(w_j)}{exp(w_j)+1} \\
\nonumber & \iff\\
\nonumber & \frac{e^{w_j}}{e^{w_j} + 1} = \frac{m_j}{m_j + n_j}\\
\nonumber & \iff\\
\nonumber & w_j = ln\frac{m_j}{n_j}
\end{align}

So we have
\[ 
{argmax} P_{\langle PF, \Pi^{\rm hard}, {\bf w}\rangle}(T; {\bf w}) = (ln\frac{m_1}{n_1}, \dots, ln\frac{m_{|PF|}}{n_{|PF|}}).
\]
\qed
\end{proof}

\BOCC
\noindent{\bf Theorem~\ref{cor:coherent-indepent-learning} \optional{cor:coherent-indepent-learning}}\
\ 
{\sl
For any coherent $\lpmln$ program template $\langle PF, \Pi^{\rm hard}\rangle$, any formula $O$ as the training data, and ${\bf W} = (w_1, \dots, w_n)$, we have
\begin{align}
\nonumber &\underset{\bf W}{argmax} P_{\langle PF, \Pi^{\rm hard}, {\bf W}\rangle}(O; {\bf W}) \\
\nonumber =& (\underset{w_1}{argmax} P_{\langle PF, \Pi^{\rm hard}, (w_1, w'_2, \dots, w'_n)\rangle}(O; w_1),\dots, \\
\nonumber &\underset{w_n}{argmax} P_{\langle PF, \Pi^{\rm hard}, (w'_1, \dots, w'_{n-1}, w_n)\rangle}(O; w_n))
\end{align}
where $w'_1, \dots, w'_n$ are any real numbers.
}

\begin{proof}
By Proposition \ref{prop:tc-independence} and Proposition \ref{prop:sm-prob-from-tc}, we have
\begin{align}
\nonumber &P_{\langle PF, \Pi^{\rm hard}, {\bf W}\rangle}(O; {\bf W})\\
\nonumber = & \underset{I\vDash O, I\in SM[{\bf \Pi}]}{\sum}\frac{1}{k\cdot \underset{pf_j\in PF}{\prod}(1+e^{w_j})}\underset{c_i\in PF}{\prod} (\frac{exp(w_i)}{exp(w_i) + 1})^{m_i(I)} \cdot (1 - \frac{exp(w_i)}{exp(w_i) + 1})^{n_i(I)} 
\end{align}
It can be seen that $P_{\langle PF, \Pi^{\rm hard}, {\bf W}\rangle}(O; {\bf W})$ is a concave function. For any $w_j\in{\bf W}$, the the value of $w_j$ that maximizes $P_{\langle PF, \Pi^{\rm hard}, {\bf W}\rangle}(O; {\bf W})$ can be obtained by solving
\[
\frac{\partial P_{\langle PF, \Pi^{\rm hard}, {\bf W}\rangle}(O; {\bf W})}{\partial w_j} = 0
\]
We have
{\footnotesize
\begin{align}
\nonumber &\frac{\partial P_{\langle PF, \Pi^{\rm hard}, {\bf W}\rangle}(O; {\bf W})}{\partial w_j} \\
\nonumber = &  \underset{I\vDash O, I\in SM[{\bf \Pi}]}{\sum}\frac{1}{k\cdot \underset{pf_j\in PF}{\prod}(1+e^{w_j})}\underset{c_i\in PF, i \neq j}{\prod} (\frac{exp(w_i)}{exp(w_i) + 1})^{m_i(I)} \cdot (1 - \frac{exp(w_i)}{exp(w_i) + 1})^{n_i(I)} \cdot\\
\nonumber &\ \ \ ((\frac{exp(w_j)}{exp(w_j) + 1})^{n_j(I)}\cdot \o{n_j}(I)\cdot (1-E)^{\o{n_j}(I)-1} + \\
\nonumber & \ \ \ \ (1-\frac{exp(w_j)}{exp(w_j)+1})^{\o{n_j}(I)}\cdot n_j(I)\cdot E^{n_j(I)-1})
\end{align}
}
where
\begin{align}
\nonumber E =&  \frac{\partial exp(w_j)/(exp(w_j) + 1)}{\partial w_j}\\
\nonumber = & \frac{exp(w_j)}{exp(w_j) + 1} - \frac{exp(w_j)}{(exp(w_j) + 1)^2}.
\end{align}
Since 
\[
 \frac{1}{k\cdot \underset{pf_j\in PF}{\prod}(1+e^{w_j})}\underset{c_i\in PF, i \neq j}{\prod} (\frac{exp(w_i)}{exp(w_i) + 1})^{m_i(I)} \neq 0,
\]
and
\begin{align}
\nonumber &\ \ \ (\frac{exp(w_j)}{exp(w_j) + 1})^{n_j(I)}\cdot \o{n_j}(I)\cdot (1-E)^{\o{n_j}(I)-1} + \\
\nonumber & \ \ \ \ (1-\frac{exp(w_j)}{exp(w_j)+1})^{\o{n_j}(I)}\cdot n_j(I)\cdot E^{n_j(I)-1} \geq 0,
\end{align}
\[
\frac{\partial P_{\langle PF, \Pi^{\rm hard}, {\bf W}\rangle}(O; {\bf W})}{\partial w_j} = 0
\]
implies
\begin{align}
\nonumber &\ \ \ (\frac{exp(w_j)}{exp(w_j) + 1})^{n_j(I)}\cdot \o{n_j}(I)\cdot (1-E)^{\o{n_j}(I)-1} + \\
\nonumber & \ \ \ \ (1-\frac{exp(w_j)}{exp(w_j)+1})^{\o{n_j}(I)}\cdot n_j(I)\cdot E^{n_j(I)-1}\\
\nonumber &= 0
\end{align}
for all $I\vDash O$. Since this equation does not involve $w_i$ for any $i\neq j$, the solution does not depend on any $w_i$ for any $i\neq j$. Consequently, we have
\begin{align}
\nonumber &\underset{\bf W}{argmax} P_{\langle PF, \Pi^{\rm hard}, {\bf W}\rangle}(O; {\bf W}) \\
\nonumber =& (\underset{w_1}{argmax} P_{\langle PF, \Pi^{\rm hard}, (w_1, w'_2, \dots, w'_n)\rangle}(O; w_1),\dots, \\
\nonumber &\underset{w_n}{argmax} P_{\langle PF, \Pi^{\rm hard}, (w'_1, \dots, w'_{n-1}, w_n)\rangle}(O; w_n)).
\end{align}
\qed
\end{proof}
\EOCC

\section{{\sc lpmln2asp} Encodings of Examples}\label{sec:encodings}

\subsection{Virus Transmitting Example}
The virus transmitting example discuss in Section \ref{ssec:learn-hypothesis} has the input program template:
\begin{lstlisting}
#domain person(X).
#domain person(Y).

@getWeight(1) has_disease(X) :- carries_virus(X).
@getWeight(2) carries_virus(Y) :- contact(X, Y), carries_virus(X).

person("A"; "B"; "C"; "D"; "E"; "F"; "G"; "H").

carries_virus("A").
contact("A", "B").
contact("B", "A").
contact("A", "E").
contact("E", "A").
contact("A", "F").
contact("F", "A").
contact("B", "C").
contact("C", "B").
contact("C", "D").
contact("D", "C").
contact("E", "F").
contact("F", "E").
contact("G", "H").
contact("H", "G").
contact("G", "A").
contact("A", "G").
\end{lstlisting}
and the following training data:
\begin{lstlisting}
:- not carries_virus("A").
:- not carries_virus("E").
:- not carries_virus("F").
:- not carries_virus("G").
:- carries_virus("B").
:- carries_virus("C").
:- carries_virus("D").
:- carries_virus("H").

:- not has_disease("A").
:- not has_disease("E").
:- has_disease("B").
:- has_disease("C").
:- has_disease("D").
:- has_disease("F").
:- has_disease("G").
:- has_disease("H").
\end{lstlisting}
The commandline
\begin{lstlisting}
python code/learn-mcsat-em.py benchmarks/virus/virus.lpmln benchmarks/virus/virus-evid.txt
\end{lstlisting}
executes gradient ascent learning with MC-ASP for 50 learning iterations and 50 MC-ASP iterations each learning iteration, which yields
\begin{lstlisting}
New weights:
Rule 1:  0.612
Rule 2:  0.574
\end{lstlisting}

\subsection{Communication Network Example}
The 10-nodes version of the communication network example discuss in Section \ref{ssec:learn-reachability} has the input program template:
\begin{lstlisting}
node(1..10).
session(1..4).
#domain session(T).
@getWeight(1) fail(1, T).
@getWeight(2) fail(2, T).
@getWeight(3) fail(3, T).
@getWeight(4) fail(4, T).
@getWeight(5) fail(5, T).
@getWeight(6) fail(6, T). 
@getWeight(7) fail(7, T).
@getWeight(8) fail(8, T).
@getWeight(9) fail(9, T).
@getWeight(10) fail(10, T).

edge(1, 2).
edge(1, 4).
edge(2, 3).
edge(4, 5).
edge(4, 6).
edge(3, 7).
edge(6, 7).
edge(5, 7).
edge(3, 8).
edge(6, 10).
edge(5, 9).

connected(X, Y, T) :- edge(X, Y), not fail(X, T), not fail(Y, T).
connected(X, Y, T) :- connected(X, Z, T), connected(Z, Y, T).
\end{lstlisting}
and the following training data:
\begin{lstlisting}
:- not connected(1, 7, 1).
:- connected(1, 8, 1).
:- not connected(1, 9, 1).
:- connected(1, 10, 1).

:- not connected(1, 7, 2).
:- not connected(1, 8, 2).
:- connected(1, 9, 2).
:- not connected(1, 10, 2).

:- not connected(1, 7, 3).
:- connected(1, 8, 3).
:- connected(1, 9, 3).
:- not connected(1, 10, 3).

:- connected(1, 7, 4).
:- not connected(1, 8, 4).
:- not connected(1, 9, 4).
:- not connected(1, 10, 4).
\end{lstlisting}

The commandline
\begin{lstlisting}
python code/learn-mcsat-em.py benchmarks/network/network.lpmln benchmarks/network/network-evid.txt
\end{lstlisting}
executes gradient ascent learning with MC-ASP for 50 learning iterations and 50 MC-ASP iterations each learning iteration, which yields
\begin{lstlisting}
New weights:
Rule 1:  -1.634
Rule 2:  1.148
Rule 3:  0.846
Rule 4:  -1.74
Rule 5:  0.184
Rule 6:  -1.0
Rule 7:  0.204
Rule 8:  0.12
Rule 9:  0.008
Rule 10:  -1.2
\end{lstlisting}

\subsection{Robot Example}
The robot discuss in Section \ref{ssec:learn-action} has the input program template:
\begin{lstlisting}
astep(0).
step(0..1).
boolean("t"; "f").
room("r1"; "r2").
instance(1..4).

#domain astep(AI).
#domain instance(ID).

% Probability Distribution
%% Entering a room fails at probability 0.2
@getWeight(1) pf1(AI, ID).
ab("enter_failed", I, ID) :- pf1(I, ID), ab(I, ID).
%% The robot drops the book at probability 0.1
@getWeight(2) pf2(AI, ID).
ab("drop_book", I, ID) :- pf2(I, ID), ab(I, ID).
%% Picking up fails at probability 0.3
@getWeight(3) pf3(AI, ID).
ab("pickup_failed", I, ID) :- pf3(I, ID), ab(I, ID).


% UEC
%% Fluents
:- not loc_robot("r1", I, ID), not loc_robot("r2", I, ID), step(I), instance(ID).
:- loc_robot("r1", I, ID), loc_robot("r2", I, ID), step(I), instance(ID).
:- not loc_book("r1", I, ID), not loc_book("r2", I, ID), step(I), instance(ID).
:- loc_book("r1", I, ID), loc_book("r2", I, ID), step(I), instance(ID).
:- not hasBook("t", I, ID), not hasBook("f", I, ID), step(I), instance(ID).
:- hasBook("t", I, ID), hasBook("f", I, ID), step(I), instance(ID).
%% Actions
:- not goto(R, "t", I, ID), not goto(R, "f", I, ID), astep(I), room(R), instance(ID).
:- goto(R, "t", I, ID), goto(R, "f", I, ID), astep(I), room(R), instance(ID).
:- not pickup_book("t", I, ID), not pickup_book("f", I, ID), astep(I), instance(ID).
:- pickup_book("t", I, ID), pickup_book("f", I, ID), astep(I), instance(ID).
:- not putdown_book("t", I, ID), not putdown_book("f", I, ID), astep(I), instance(ID).
:- putdown_book("t", I, ID), putdown_book("f", I, ID), astep(I), instance(ID).

% Effect of Actions
loc_robot(R, I+1, ID) :- goto(R, "t", I, ID), not ab("enter_failed", I, ID), instance(ID).
loc_book(R, I, ID) :- loc_robot(R, I, ID), hasBook("t", I, ID), instance(ID).
hasBook("t", I+1, ID) :- pickup_book("t", I, ID), loc_robot(R, I, ID), loc_book(R, I, ID), 
                         not ab("pickup_failed", I, ID), instance(ID).
hasBook("f", I+1, ID) :- putdown_book("t", I, ID), instance(ID).
hasBook("f", I+1, ID) :- ab("drop_book", I, ID), instance(ID).

% Frame Axioms
loc_robot(R, I+1, ID) :- loc_robot(R, I, ID), astep(I), instance(ID), not not loc_robot(R, I+1, ID).
loc_book(R, I+1, ID) :- loc_book(R, I, ID), astep(I), instance(ID), not not loc_book(R, I+1, ID).
hasBook(B, I+1, ID) :- hasBook(B, I, ID), astep(I), instance(ID), not not hasBook(B, I+1, ID).

% No Concurrency
:- goto(R1, "t", I, ID), goto(R2, "t", I, ID), astep(I), instance(ID), R1 != R2.
:- goto(R, "t", I, ID), pickup_book("t", I, ID), room(R), astep(I), instance(ID).
:- goto(R, "t", I, ID), putdown_book("t", I, ID), room(R), astep(I), instance(ID).
:- pickup_book("t", I, ID), putdown_book("t", I, ID), astep(I), instance(ID).

% Initial state and actions are exogenous
loc_robot("r1", 0, ID) :- instance(ID), not loc_robot("r2", 0, ID).
loc_robot("r2", 0, ID) :- instance(ID), not loc_robot("r1", 0, ID).

loc_book("r1", 0, ID) :- instance(ID), not loc_book("r2", 0, ID).
loc_book("r2", 0, ID) :- instance(ID), not loc_book("r1", 0, ID).

hasBook("t", 0, ID) :- instance(ID), not hasBook("f", 0, ID).
hasBook("f", 0, ID) :- instance(ID), not hasBook("t", 0, ID).

goto(R, "t", I, ID) :- room(R), astep(I), instance(ID), not goto(R, "f", I, ID).
goto(R, "f", I, ID) :- room(R), astep(I), instance(ID), not goto(R, "t", I, ID).

pickup_book("t", I, ID) :- astep(I), instance(ID), not pickup_book("f", I, ID).
pickup_book("f", I, ID) :- astep(I), instance(ID), not pickup_book("t", I, ID).

putdown_book("t", I, ID) :- astep(I), instance(ID), not putdown_book("f", I, ID).
putdown_book("f", I, ID) :- astep(I), instance(ID), not putdown_book("t", I, ID).

ab(I, ID) :- instance(ID), astep(I).
\end{lstlisting}
and the following training data:
\begin{lstlisting}
:- not loc_robot("r1", 0, 1).
:- not loc_book("r2", 0, 1).
:- not hasBook("f", 0, 1).
:- not goto("r2", "t", 0, 1).
:- not loc_robot("r1", 1, 1).

:- not loc_robot("r1", 0, 2).
:- not loc_book("r2", 0, 2).
:- not hasBook("f", 0, 2).
:- not goto("r2", "t", 0, 2).
:- not loc_robot("r2", 1, 2).

:- not loc_robot("r1", 0, 3).
:- not loc_book("r2", 0, 3).
:- not hasBook("f", 0, 3).
:- not goto("r2", "t", 0, 3).
:- not loc_robot("r2", 1, 3).

:- not loc_robot("r1", 0, 4).
:- not loc_book("r2", 0, 4).
:- not hasBook("f", 0, 4).
:- not goto("r2", "t", 0, 4).
:- not loc_robot("r2", 1, 4).

:- not loc_robot("r1", 0, 5).
:- not loc_book("r1", 0, 5).
:- not hasBook("f", 0, 5).
:- not pickup_book("t", 0, 5).
:- not hasBook("f", 1, 5).

:- not loc_robot("r1", 0, 6).
:- not loc_book("r1", 0, 6).
:- not hasBook("f", 0, 6).
:- not pickup_book("t", 0, 6).
:- not hasBook("f", 1, 6).

:- not loc_robot("r1", 0, 7).
:- not loc_book("r1", 0, 7).
:- not hasBook("f", 0, 7).
:- not pickup_book("t", 0, 7).
:- not hasBook("t", 1, 7).

:- not loc_robot("r1", 0, 8).
:- not loc_book("r1", 0, 8).
:- not hasBook("f", 0, 8).
:- not pickup_book("t", 0, 8).
:- not hasBook("t", 1, 8).

:- not loc_robot("r1", 0, 9).
:- not hasBook("t", 0, 9).
:- not hasBook("f", 1, 9).

:- not loc_robot("r1", 0, 10).
:- not hasBook("t", 0, 10).
:- not hasBook("t", 1, 10).

:- not loc_robot("r1", 0, 11).
:- not hasBook("t", 0, 11).
:- not hasBook("t", 1, 11).

:- not loc_robot("r1", 0, 12).
:- not hasBook("t", 0, 12).
:- not hasBook("t", 1, 12).
\end{lstlisting}

The commandline
\begin{lstlisting}
python code/learn-mcsat-em.py benchmarks/robot/robot.lpmln benchmarks/robot/robot-evid.txt
\end{lstlisting}
executes gradient ascent learning with MC-ASP for 50 learning iterations and 50 MC-ASP iterations each learning iteration, which yields
\begin{lstlisting}
New weights:
Rule 1: -1.084
Rule 2:  -1.064
Rule 3:  -0.068
\end{lstlisting}

\BOCC
\subsection{Comparison between ProbLog and Our Prototype System on Network Failure Example}
ProbLog's performance on weight learning depends on the tightness of the input program. We observed that for many tight programs, Problog appears to have better scalability than our prototype lpmln weight learning system. However, Problog does not have a consistently reasonable performance on non-tight programs, such as the program in \ref{ssec:learn-reachability}, possibly due to that it has to convert the input program into weighted Boolean formulas, which is expensive for non-tight programs.  We can identify many graph instances of the network failure example discussed in Section \ref{ssec:learn-reachability} where our prototype system outperforms Problog, as the density of the graph gets higher. For example, consider the graph in Figure 1. With the nodes fixed, as we add more edges to the graph, we eventually hit a point when Problog does not return result within a reasonable time.

Below is the statistics of several instances.

\bigskip

{\footnotesize
\centering
\begin{tabular}{| c | c c |}
\hline
 {\bf \# Edges} & {\bf {\sc lpmln-learn}} & {\bf {\sc ProbLog}} \\
 \hline
 $10 (Original)$ & 351.237s & 2.565s \\
 $14$ & 476.656s & 2.854s \\
 $15$ & 740.656s & $>$ 20 min \\
 $20$ & 484.348s & $>$ 20 min \\
$40$ & 304.407s & $>$ 20 min \\
 \hline
\end{tabular}
}

\bigskip

The input files to Problog consists of two parts: edge lists and the part that defines the node failure rates and connectivity. The latter is shown below:
\begin{lstlisting}
t(_)::fail(1).
t(_)::fail(2).
t(_)::fail(3).
t(_)::fail(4).
t(_)::fail(5).
t(_)::fail(6).
t(_)::fail(7).
t(_)::fail(8).
t(_)::fail(9).
t(_)::fail(10).

connected(X, Y) :- edge(X, Y), not fail(X), not fail(Y).
connected(X, Y) :- connected(X, Z), connected(Z, Y).
\end{lstlisting}

The edge list differs from instance to instance, as an example, the edge list for the instance with 40 edges is
\begin{lstlisting}
edge(1, 2).
edge(1, 4).
edge(2, 3).
edge(4, 5).
edge(4, 6).
edge(3, 7).
edge(6, 7).
edge(5, 7).
edge(3, 8).
edge(6, 10).
edge(5, 9).
edge(1, 3).
edge(1, 5).
edge(1, 6).
edge(1, 7).
edge(1, 8).
edge(1, 9).
edge(1, 10).
edge(2, 4).
edge(2, 5).
edge(2, 6).
edge(2, 7).
edge(2, 8).
edge(2, 9).
edge(2, 10).
edge(3, 4).
edge(3, 5).
edge(3, 6).
edge(3, 8).
edge(3, 9).
edge(3, 10).
edge(4, 1).
edge(4, 2).
edge(4, 3).
edge(4, 4).
edge(4, 7).
edge(4, 8).
edge(4, 9).
edge(4, 10).
\end{lstlisting}

The command line is
\begin{lstlisting}
problog lfi prog.pl evid.pl -O prog-learned.pl -n 1
\end{lstlisting}
 
The difference between the instance with 14 edges and 15 edges is that one edge {\tt edge(4, 1)} is added to the one with 15 edges. Together with {\tt edge(1, 4)}, it forms a cycle in the graph. 

\subsection{Comparison between MLN and $\lpmln$ weights learning on the Virus Example}
For the program in Section \ref{ssec:learn-hypothesis}, we can perform weight learning under both $\lpmln$ and MLN, and compare the marginal probability of people carrying virus under $\lpmln$ and MLN, with respective weight learned. Consider a specific instance of the problem, for which the contact relation is shown as the following graph
\begin{center}
\includegraphics[width=0.5\linewidth]{Virus-ijcai-rebuttal.png}
\end{center}
where {\tt A} is the person who initially carries the virus and triangle shaped nodes represents people who carries virus in the evidence. Note that the cluster consists of {\tt E}, {\tt F}, {\tt G} has no contact with the cluster consists of {\tt A}, {\tt B}, {\tt C}, {\tt D}.
For MLN learning, we use Alchemy, below is the result.

\vspace{0.2cm}
{\footnotesize
\centering
\begin{tabular}{| c | c c c|}
\hline
 {\bf Person} & {\bf {\sc MLN}} & {\bf $\lpmln$} & carries\_virus (ground truth)\\
 \hline
 $B$ & 0.823968 & 0.6226904833 & Y \\
 $C$ & 0.813969 & 0.6226904833 & Y \\
 $D$ & 0.818968 & 0.6226904833 & N \\
 $E$ & 0.688981 & 0 & N \\
$F$ & 0.680982 & 0 & N \\
$G$ & 0.680982 & 0 & N \\
 \hline
\end{tabular}
}

\vspace{0.2cm}

As can be seen from the table, under MLN, each of {\tt E}, {\tt F}, {\tt G} has a unreasonably high probability of carrying virus regardless of the fact that they are not even contacted with any one carrying the virus. 
\EOCC

\end{document}